%% file: main.tex
\documentclass[10pt,letterpaper]{article}
\input{format.tex}

\input{macros.tex}

\usepackage{graphicx}
\usepackage{booktabs}
\usepackage{longtable}
\usepackage{placeins}
\usepackage{needspace}
\usepackage{url}
\usepackage[colorlinks=true,linkcolor=red,citecolor=blue,urlcolor=blue]{hyperref}
\usepackage[capitalize,noabbrev]{cleveref}

\usepackage{physics}

\crefname{proposition}{Proposition}{Propositions}
\crefname{theorem}{Theorem}{Theorems}
\crefname{lemma}{Lemma}{Lemmas}
\crefname{corollary}{Corollary}{Corollaries}
\crefname{remark}{Remark}{Remarks}
\crefname{example}{Example}{Examples}
\crefname{definition}{Definition}{Definitions}

\graphicspath{{figs/}}

\title{Rethinking Least-Core Computation\\in Contextual-Distractor Games}

\input{authors.tex}
\date{}

\begin{document}

\maketitle

\input{sections/00_abstract}
\input{sections/01_intro}
\input{sections/02_related}

\input{sections/03_background}
\input{sections/04_method}
\input{sections/05_empirical_validation}
\input{sections/06_distractors}
\input{sections/07_applications}
\input{sections/08_conclusion}

\section*{Acknowledgments}
This research was partially supported by JST PRESTO Grant Number JPMJPR24K4,
JST BOOST Program Grant Number JPMJBY24C6, and JSPS KAKENHI Grant Number
JP23KK0208 and 26K02996.

\bibliography{refs}
\bibliographystyle{plainnat}

\clearpage
\appendix
\input{sections/A_appendix}

\end{document}

%% file: format.tex
\usepackage{times}
\usepackage[letterpaper,textwidth=5.5in,textheight=9in,top=0.75in]{geometry}
\usepackage[authoryear,round]{natbib}
\setcitestyle{citesep={;},aysep={,},yysep={;}}
\usepackage{titlesec}
\titleformat{\section}{\large\bfseries\raggedright}{\thesection}{1em}{}
\titleformat{\subsection}{\normalsize\bfseries\raggedright}{\thesubsection}{1em}{}
\titleformat{\subsubsection}{\normalsize\bfseries\raggedright}{\thesubsubsection}{1em}{}
\titlespacing*{\section}{0pt}{2ex plus .5ex minus .2ex}{1.5ex plus .3ex minus .2ex}
\titlespacing*{\subsection}{0pt}{1.8ex plus .5ex minus .2ex}{.8ex plus .2ex}
\titlespacing*{\subsubsection}{0pt}{1.5ex plus .5ex minus .2ex}{.5ex plus .2ex}
\titleformat{\paragraph}[runin]{\normalsize\bfseries}{}{0pt}{}
\titlespacing*{\paragraph}{0pt}{1.5ex plus .5ex minus .2ex}{1em}
\makeatletter
\renewcommand{\normalsize}{%
  \@setfontsize\normalsize{10pt}{11pt}%
  \abovedisplayskip 7pt plus 2pt minus 5pt
  \belowdisplayskip\abovedisplayskip
  \abovedisplayshortskip 0pt plus 3pt
  \belowdisplayshortskip 4pt plus 3pt minus 3pt}
\makeatother
\normalsize

%% file: macros.tex
\usepackage{amsmath,amssymb,amsthm,bm,aliascnt,xcolor,wrapfig,caption,iftex}

\definecolor{KeraGreen}{RGB}{0,100,0}

\theoremstyle{definition}
\newtheorem{definition}{Definition}[section]
\newaliascnt{proposition}{definition}
\newtheorem{proposition}[proposition]{Proposition}
\aliascntresetthe{proposition}
\newaliascnt{theorem}{definition}

\aliascntresetthe{theorem}
\newaliascnt{lemma}{definition}
\newtheorem{lemma}[lemma]{Lemma}
\aliascntresetthe{lemma}
\newaliascnt{corollary}{definition}
\newtheorem{corollary}[corollary]{Corollary}
\aliascntresetthe{corollary}
\newaliascnt{remark}{definition}

\aliascntresetthe{remark}
\newaliascnt{example}{definition}

\aliascntresetthe{example}
\newaliascnt{fact}{definition}

\aliascntresetthe{fact}
\usepackage{thm-restate}

\DeclareMathOperator*{\argmin}{arg\,min}

\DeclareMathOperator{\Cov}{Cov}

\newcommand{\R}{\mathbb{R}}
\newcommand{\E}{\mathbb{E}}
\newcommand{\Prob}{\mathrm{Pr}}
\newcommand{\Var}{\mathrm{Var}}
\newcommand{\KL}{\mathrm{KL}}

\newcommand{\norm}[1]{\lVert#1\rVert}

\newcommand{\cS}{\mathcal{S}}

\newcommand{\ind}[1]{\bm{1}_{#1}}
\newcommand{\one}{\bm{1}}

\newcommand{\excess}[2]{e_{#1}(#2)}     %
\newcommand{\gw}[2]{u_{#1}^{\eta}(#2)} %
\newcommand{\gibbsdist}{\lambda_{x}^{\eta}}

\newcommand{\lcv}{\varepsilon^{*}}
\newcommand{\slcv}{\varepsilon_{\eta}}
\newcommand{\etahat}{\hat{\varepsilon}}

%% file: authors.tex
\author{%
Hiroshi Kera\thanks{Corresponding author.}\\
\normalsize\normalfont Chiba University; National Institute of Informatics\\
\normalsize\normalfont\texttt{kera@chiba-u.jp}
\\[1ex]
Toshinori Yamauchi\\
\normalsize\normalfont Chiba University, Graduate School of Informatics\\
\normalsize\normalfont\texttt{t.yamauchi@chiba-u.jp}
\\[1ex]
Sai Ganesh Nagarajan\\
\normalsize\normalfont University of Southern Denmark, IMADA\\
\normalsize\normalfont\texttt{sgnagarajan@imada.sdu.dk}
}
\hypersetup{pdfauthor={Hiroshi Kera, Toshinori Yamauchi, Sai Ganesh Nagarajan}}

%% file: sections/00_abstract.tex
\begin{abstract}
Game-theoretic attribution explains a model by assigning credit to its
features or training examples. The least core has attracted interest
as an alternative to Shapley-style averaging because it can expose
players that cause substantial harm in rare, high-value contexts.
However, least-core allocations are generally nonunique, and the
choice of allocation can affect the resulting explanation.
In this study, we investigate how payoff selection and coalition
sampling affect least-core attribution. Our experiments show that
selector choice matters for distinguishing useful and harmful
contributions, and that sampling can degrade harmful-player
identification across the tested selectors even when useful players
remain well identified. These observations motivate efficient
computation with all coalition constraints and a well-defined selector.
We introduce entropic least core (ELC), a smooth approximation whose
unique minimizer follows a continuous path along the temperature to
the nucleolus, a classical refinement of the least core. 
Our experiments show that ELC approximates the nucleolus faster than
an LP-based nucleolus solver while retaining small payoff errors,
with further GPU acceleration at larger problem sizes. In the tested
full-coalition contextual-distractor games, ELC matches the minimum-norm
selector in identification accuracy and more accurately ranks distractors
by harm.
Code and reference results are available online.\footnote{\url{https://github.com/HiroshiKERA/rethinking-least-core}}
\end{abstract}

%% file: sections/01_intro.tex
\section{Introduction}
\label{sec:intro}

A fundamental building block of machine-learning model explanation is scoring the contribution of each feature, training example, or agent to the model's output. Game-theoretic attribution has provided a framework for measuring contributions through a coalitional game: \textit{players} work together as a coalition, and the central task is to allocate the resulting value fairly among the players. In particular, the Shapley value~\citep{shapley1953} is a widely used index. The Shapley value credits each player by its average marginal contribution: it measures the change in coalition value when the player is added, and takes a weighted average over all coalitions not containing that player. Applications include feature attribution, data valuation, and computer-vision explanations~\citep{strumbelj2010,lundberg2017,sundararajan2017,ghorbani2019,covert2020,jethani2022,sumiyasu2024,yamauchi2026,pred-diff}.

The least core~\citep{maschler1979} has recently been studied as a useful game-theoretic index for model explanation~\citep{yan2021,gemp2024,nagarajan2026,harris2026core}. Rather than averaging the marginal contribution, the least core seeks a stable payoff allocation that minimizes the maximum underpayment to coalitions. The critical difference from the Shapley value has recently been formalized by \citet{nagarajan2026} through \textit{contextual distractors}, players that have a minor positive impact in most coalitions but cause substantial harm in rare, high-value coalitions. In their experiments, the least core assigned negative payoffs to them by accounting for the worst case, whereas the Shapley value failed, assigning positive payoffs because averaging lets frequent minor positive contributions outweigh rare but substantial harm. This distinction is relevant to understanding model failures in rare, high-impact contexts.

However, the aforementioned studies provide limited discussion of the choice of \textit{selector} for the least core. Least-core computation is formulated as a linear program (LP) with exponentially many constraints, each corresponding to a coalition. The minimizing payoff vectors are generally not unique, forming a polytope. A selector is a map that chooses one payoff allocation from the polytope, and this choice affects the resulting explanation. However, many works use the minimum-norm selector as a de facto standard~\citep{yan2021,harris2026core,nagarajan2026}, while its suitability for contextual-distractor identification relative to other selectors remains insufficiently examined. Approximate computation methods~\citep{yan2021,gemp2024} do not discuss the final payoff allocation, either; \cite{yan2021} only bounds the optimality gap of minimum value, not the minimizers. \citet{gemp2024} did not specify the selector, and thus the final payoff allocation is implicitly determined by the optimization process and its parameters.

In this study, we investigate least-core computation with a focus on payoff selection. Particularly, we focus on the effect of the selector choice and coalition sampling on the contextual-distractor setting.
Our key observations are twofold: (i) the choice of selector is essential for successful retrieval of contextual distractors, and (ii) coalition sampling particularly degrades their retrieval for standard selectors.
For the former, we found the nucleolus selector to be the most successful, but expensive to compute.
The latter is particularly important because existing acceleration methods~\citep{yan2021,gemp2024} rely on coalition sampling.
These observations motivate a more tractable approximation of the nucleolus with full constraints. We introduce the \textit{entropic least core} (ELC), a smoothed approximation of the least core that replaces the maximum excess with a log-sum-exp objective. At temperature $\eta>0$, the ELC objective is convex and differentiable, with a unique minimizer, and can be optimized by projected gradient descent with GPU acceleration.

Our contributions are as follows:
\begin{itemize}
  \item We revisit least-core computation from the payoff allocation perspective and provide empirical observations to inform its design. We show the substantial impact of selector choice and the potential unreliability of coalition sampling for retrieving contextual distractors.
  \item We propose the $\eta$-entropic least core ($\eta$-ELC) as a smoothed approximation of the least core. Its unique payoff varies continuously with temperature and converges to the nucleolus as $\eta\to0$. Our experiments show that ELC approximates the nucleolus with small payoff errors in less solver time than the successive-LP nucleolus solver.
  \item Our experiments compare Shapley attribution and least-core methods with different selectors. (i) In the tested games, all methods identify useful players perfectly with full coalitions. Under sampling, useful-player identification remains comparatively strong, although Shapley performance declines. (ii) With full coalitions, the minimum-norm selector and ELC both attain perfect average precision for contextual-distractor identification, but ELC more accurately ranks distractors by harm. ELC thus provides more fine-grained attribution in these games.
\end{itemize}

%% file: sections/02_related.tex
\section{Related work}
\label{sec:related}

The Shapley value~\citep{shapley1953} is a widely used game-theoretic
measure of contribution in model explanation~\citep{strumbelj2010,lundberg2017},\footnote{Further Shapley-related work is discussed in \cref{app:additional-related}.}
whereas the least core~\citep{maschler1979} has received comparatively
limited attention.
\citet{yan2021} introduced the core and least core for model explainability.
Toward computational tractability, they showed that solving an LP with sampled coalition constraints may attain \textit{probably approximate} least core, i.e., good approximation of the minimal maximal excess in a high probability.
\citet{gemp2024} subsequently developed LP-free methods based on cyclic
projections and a Core Lagrangian formulation with stochastic updates.
Both studies pursued more efficient least-core computation and compared
its attributions with Shapley values in machine-learning applications.
\citet{nagarajan2026} identified a concrete source of advantage for the
least core through \emph{contextual distractors}: groups of features that
help in less informative contexts but strongly harm important ones.
They established conditions under which Shapley and other semivalues
assign more credit to these groups than the least core, and demonstrated
this effect in explaining biased image classifiers.

Our study addresses computational aspects of the least-core in this contextual-distractor setting---is the minimum-norm selector, the de facto standard selector, the most useful for retrieving contextual distractors, and how much does the sampling strategy of~\citet{yan2021} affect retrieval?
As \citet{gemp2024,harris2026core} focus on nonnegative payoffs and cannot assign negative payoffs to contextual distractors, their methods are outside the scope of this study.

%% file: sections/03_background.tex
\section{Preliminaries}
\label{sec:background}

Let $G=(N,v)$ be a coalitional game with player set
$N=\{1,\ldots,n\}$ and $v:2^N\to\R$, $v(\emptyset)=0$.
An attribution is represented by a vector $x\in\R^n$, where $x_i$ is the
credit assigned to player $i$.
We call such vectors \emph{payoff vectors}.
Write $x(S):=\sum_{i\in S}x_i$ for the total credit of a coalition $S$,
and $\excess{S}{x}:=v(S)-x(S)$ for its \emph{excess} (underpayment)
relative to the worth $v(S)$.
A payoff vector with $x(N)=v(N)$ is called \emph{efficient}.

\subsection{Shapley value and least core}

We compare two payoff rules for the same coalitional game: the Shapley
value averages marginal contributions, whereas the least core controls
the worst coalition-level underpayment.

\begin{definition}[Shapley value]
The Shapley value is the payoff vector $x$ with
\begin{equation}
  x_i
  =
  \sum_{S\subseteq N\setminus\{i\}}
  \frac{|S|!(n-|S|-1)!}{n!}
  \qty(v(S\cup\{i\})-v(S)).
\end{equation}
\end{definition}
In words, it averages the marginal contribution of player $i$ over all
contexts $S \subseteq N \setminus \{i\}$.
In contrast, the least core focuses on the worst excess across
coalitions.
Throughout the paper, least-core constraints range over the proper
coalition family $\cS=2^N\setminus\{\emptyset,N\}$, and we write
$M=|\cS|=2^n-2$.
For efficient payoffs, both the empty and grand coalitions have zero
excess. Excluding both allows the least-core value to be negative.

\begin{definition}[$\epsilon$-core]
For $\epsilon \in \R$, the $\epsilon$-core is the set
$\operatorname{Core}(G,\epsilon) \subset \R^n$ of payoffs that satisfy
the following constraint:
\begin{equation}
  x(S)\ge v(S)-\epsilon\ \text{ for all } S\in\cS,
  \qquad x(N)=v(N).
  \label{eq:core-constraint}
\end{equation}
The smallest $\epsilon$ with $\operatorname{Core}(G,\epsilon)\neq\emptyset$
is the \emph{least-core value} $\lcv$.
The set $\mathrm{LC}(\lcv):=\operatorname{Core}(G,\lcv)$ is the
\emph{least core}.
\end{definition}

The least core is a set of payoff vectors, while a least-core payoff is
one point in that set.
The least-core value and one least-core payoff are computed by the
linear program
\begin{equation}\label{eq:lclp}
  \min_{x\in\R^n,\,\epsilon\in\R}\ \epsilon \quad \text{s.t.} \quad x(N)=v(N)\ \text{ and }\ x(S)\ge v(S)-\epsilon,\ \forall S\in\cS.
\end{equation}
Equivalently, it minimizes the maximum excess
$\max_{S\in\cS}\excess{S}{x}$ over efficient payoffs.
This worst-coalition criterion distinguishes the least core from the
Shapley value.

\subsection{Selectors for the least core}
\label{sec:bg-selector}
The least-core value $\lcv$ is unique, but the minimizing payoff vector need not be unique; $x \in \mathrm{LC}(\lcv)$ is a least-core payoff. Without any specification, the returned payoff depends on the solver and the parameters.
A \emph{selector} is a rule that returns one least-core payoff from
the least-core face.
There are several natural selectors~\citep{nagarajan2026}.
We describe two of them here.

\paragraph{Minimum-norm selector.}
The minimum-norm selector breaks ties by the $\ell_2$ norm,
\begin{equation}
  x_{\mathrm{mn}}
  := \sigma_{\mathrm{mn}}(\mathrm{LC}(\lcv))
  := \argmin_{x\in\mathrm{LC}(\lcv)}\ \norm{x}_2^2.
\end{equation}
The objective is strictly convex, so the minimizer is unique.
Intuitively, among least-core payoffs with the same stability guarantee,
this selector favors the point closest to the origin.

\paragraph{Nucleolus selector.}
The nucleolus\footnote{Strictly, we consider the \emph{prenucleolus}:
we impose efficiency without individual rationality
($x_i\ge v(\{i\})$), allowing negative credit.
This follows \citet{yan2021}.}~\citep{schmeidler1969,maschler1979} refines
the least core by minimizing the second-largest excess, then the third,
and so on, without worsening the preceding minima.
Fix an enumeration $S_1,\ldots,S_M$ of $\mathcal S$.
For each candidate payoff $x$, let $(\sigma_1(x),\ldots,\sigma_M(x))$
be a permutation of $\{1,\ldots,M\}$ that sorts the excesses:
\[
  e_{S_{\sigma_1(x)}}(x)\ge e_{S_{\sigma_2(x)}}(x)\ge\cdots\ge e_{S_{\sigma_M(x)}}(x).
\]
The nucleolus lexicographically minimizes this vector over efficient payoffs (cf. \cref{app:nucleolus}).

The min-norm selector is intuitive and popular in literature~\citep{yan2021,harris2026core,nagarajan2026}.
The nucleolus is more aligned with the least-core idea as it optimizes the excess profile beyond the worst excess, while \citet{yan2021} showed that their sampling approach requires exponential queries to give a good approximation of the nucleolus.
In \citet{gemp2024}, the returned payoff vector is determined
implicitly by the iterative procedure and can depend on hyperparameters,
initialization, sampling, and stopping. It is worth noting that their method restricts the payoff to the simplex, and thus cannot assign negative payoffs to contextual distractors.

\subsection{Contextual distractors and a motivating example}
\label{sec:distractors}
\label{sec:small-freerider}

We now introduce the contextual distractors and a motivating example that highlights the effect of the selector choice on the attribution.

\begin{definition}[Contextual distractor~\citep{nagarajan2026}]
\label{def:contextual-distractor}
Let $D\subset N$ be a nonempty group. For a context $T\subseteq N\setminus D$, write
$\Delta_D(T)=v(T\cup D)-v(T)$.
For $\gamma>0$, $D$ is a $\gamma$-contextual distractor if
$\Delta_D(T)\le-\gamma$ for some such $T$.
\end{definition}
\citet{nagarajan2026} discussed when the least core and the Shapley value assign different attribution scores. A critical condition is that $\Delta_D(T) \le -\gamma$ for $v(T) \ge \tau$ for some $\tau > 0$ and $\Delta_D(T) > 0$ otherwise. Namely, the contextual distractor $D$ is harmful only when $T$ is a (rare), high-value context, and otherwise harmless.

Throughout this paper, we focus on this case. The following game illustrates the contextual distractors $D$ for useful players $U$ and the importance of the selector choice.

\paragraph{A simple contextual-distractor game.}
Let $U$ contain useful players with additive contributions $m_i$, and let $D$ contain distractors:
\begin{equation}
 v(S)=\sum_{i\in S\cap U}m_i+a|S\cap D|
       +b\,\mathbf1[U\subseteq S]\,\mathbf1[S\cap D=\emptyset],
 \label{eq:simple-freerider}
\end{equation}
where $\mathbf1[A]$ equals $1$ if condition $A$ holds and $0$ otherwise, and $0< a|D| < b$. Each distractor contributes $a>0$ in ordinary coalitions, but destroys the bonus $b$ when all useful players are present without any distractor. Formally, $D$ is a contextual distractor with $\Delta_D(U) \le a|D| - b$.

Consider four useful players with $m=(6,8,12,14)$,
two distractors, $a=5$, and $b=60$. \Cref{tab:freerider} compares the payoff of the least core with different selectors. \textit{LP} denotes the vertex payoff returned by the LP solver, which is determined by the solver and the parameters. The Shapley value and our method ELC are also given for reference. As can be seen, the Shapley value fails to identify the contextual distractors, assigning positive credit to them, but the least core methods correctly assign negative credit to at least one of them. Particularly, the LP with the min-norm and nucleolus selectors successfully assign negative credit to both distractors. The latter is more accurate in the positive credit assignment to the useful players; it retrieves $m_i$ with an offset of $7.5$ to each useful player. This common offset is consistent with the efficiency constraint $x(N)=v(N)$ and does not affect useful-player ordering.
\begin{table}[t]
  \caption{Attribution in a six-player contextual-distractor game.
  Shapley assigns positive credit to both distractors. Least-core
  allocations assign negative total credit to the pair, but differ in
  individual signs and the relative credit of useful players.}
  \label{tab:freerider}
  \centering
  \small\setlength{\tabcolsep}{4pt}
  \begin{tabular}{lrrrrr}
  \toprule
  Player & Shapley & LP & LP + min-norm & LP + nucleolus & ELC \\
  \midrule
  useful 1 ($m_1=6$)  & $7$ & $6$ & $17.5$ & $13.5$ & $13.5$ \\
  useful 2 ($m_2=8$)  & $9$ & $8$ & $17.5$ & $15.5$ & $15.5$ \\
  useful 3 ($m_3=12$) & $13$ & $12$ & $17.5$ & $19.5$ & $19.5$ \\
  useful 4 ($m_4=14$) & $15$ & $44$ & $17.5$ & $21.5$ & $21.5$ \\
  distractor 1 & $3$ & $5$ & $-10$ & $-10$ & $-10$ \\
  distractor 2 & $3$ & $-25$ & $-10$ & $-10$ & $-10$ \\
  \midrule
  Max.\ excess & $56$ & $\mathbf{30}$ & $\mathbf{30}$ & $\mathbf{30}$ & $\mathbf{30}$ \\
  \bottomrule
  \end{tabular}
  \end{table}

%% file: sections/04_method.tex
\section{Entropic least core}
\label{sec:method}

\subsection{Definition and properties}

We introduce entropic least core (ELC), which is a \textit{smoothed}
least-core payoff and approaches the least-core face in the limit of
small temperature.

The least-core value minimizes the maximum excess over efficient payoffs:
\begin{equation}
  \lcv = \min_{x:\, x(N)=v(N)}\ \etahat(x),\quad \text{ where }\quad  \etahat(x) := \max_{S\in\cS}\ \excess{S}{x}.
\end{equation}
ELC replaces this hard maximum by the log-sum-exp approximation
\begin{equation}\label{eq:smoothed-excess}
  \slcv(x) = \eta \log \sum_{S\in\cS} \gw{S}{x},\quad \text{ where }\quad \gw{S}{x} := \exp\bigl(\excess{S}{x}/\eta\bigr),
\end{equation}
where $\eta>0$ is the temperature parameter.
Letting $M := |\cS|$, the log-sum-exp bound gives
\begin{equation}
  \etahat(x) \le \slcv(x) \le \etahat(x) + \eta\log M.
\end{equation}
Thus $\slcv(x) \to \etahat(x)$ as $\eta\to 0^{+}$.

\begin{definition}[Entropic least core]
  The $\eta$-entropic least core ($\eta$-ELC) is the payoff vector that
  minimizes the smoothed maximum excess over efficient payoff vectors:
  \begin{equation}\label{eq:elc}
    x_\eta
    := \argmin_{x:\, x(N)=v(N)}\ \slcv(x).
  \end{equation}
\end{definition}

\paragraph{Entropic view.} The replacement of the hard max by the log-sum-exp function has a natural interpretation of introducing an entropy regularization term. Let $\Delta$ be the simplex over the coalition family $\cS$, and let $w=(w_S)_{S\in\cS}\in\Delta$ be a reference distribution with full support. The variational representation of the log-sum-exp function suggests that for any $a\in\R^{M}$ and $\eta>0$,
\begin{equation}
  \max_{\lambda\in\Delta}
  \Bigl\{
    \sum_{S\in\cS} \lambda_S a_S
    - \eta\,\KL(\lambda \| w)
  \Bigr\}
  =
  \eta \log \sum_{S\in\cS} w_S\, e^{a_S/\eta}.
\end{equation}
The left-hand side is a maximization over the simplex with regularization by the Kullback--Leibler divergence, called entropic regularization in the literature of optimal transport~\citep{cuturi2013sinkhorn}.
The maximizer is unique and is the Gibbs distribution
$\lambda_S\propto w_S e^{a_S/\eta}$.
Setting $a_S=\excess{S}{x}$ and $w_S=1/M$ for all $S\in\cS$ gives the ELC objective
up to the additive constant $-\eta\log M$, giving the same
minimizer.

\paragraph{Why not regularize the least core directly?}
A natural alternative is to compute the least-core polytope and then select a
canonical point using a strongly convex or entropic regularizer. This does not
remove the main computational bottleneck: the least core contains one
constraint per coalition,
\[
    \excess{S}{x}\le \varepsilon
    \qquad
    \forall S\in\cS,
\]
and separation at a candidate $(x,\varepsilon)$ requires solving $\max_{S\in\cS}\excess{S}{x}$. Efficient optimization or projection over such an exponentially described
polytope therefore relies on an efficient separation oracle
\citep[pp.~102--132]{grotschellovaszschrijver1988}, while core and
$\varepsilon$-core related problems are intractable for general compactly
represented cooperative games without additional structure
\citep{deng1994complexity}. Thus, regularizing the allocation changes how one
selects among feasible least-core solutions, but leaves the hard coalition
search intact. Entropic smoothing instead trades this combinatorial search
for stochastic estimation, replacing the exact worst-coalition problem by
sampling the influence of high-excess coalitions.

The following proposition characterizes the $\eta$-ELC and its limit,
using the exponential-penalty result of \citet{alvarez2000}
to establish convergence to the nucleolus.
\begin{restatable}[ELC properties]{proposition}{elcproperties}
\label{thm:elc-properties}
For $n\ge2$ and $\cS=2^N\setminus\{\emptyset,N\}$, the following hold.
 \begin{enumerate}
  \item For every $\eta>0$, the $\eta$-ELC is a singleton; namely,
        \cref{eq:elc} has a unique minimizer.
  \item $0\le\etahat(x_\eta)-\lcv\le\eta\log M$.
        As $\eta\to0^+$, $\slcv(x_\eta)\to\lcv$.
  \item The exact ELC path converges to the nucleolus
        \(x_{\mathrm{nuc}}\); that is,
        \(x_\eta\to x_{\mathrm{nuc}}\) as \(\eta\to0^+\).
 \end{enumerate}
\end{restatable}
The proof is in \cref{app:zero-temperature-selector}.
The first two statements say that ELC is a principled smooth
approximation to the least core: every positive temperature gives a unique payoff
whose worst excess is within $\eta\log M$ of the least-core value.
The third statement identifies the induced selector: it converges to the
nucleolus, the classical lexicographic refinement of the least core.

\paragraph{Selector view.} From this perspective, the KL-regularized
objective is not only a smooth surrogate for the least-core value; it equips the approximate least core with a selector, inducing the unique payoff for each $\eta > 0$. 
As the approximation vanishes with $\eta \to 0^+$, this payoff smoothly converges to the nucleolus. Quadratic regularization instead selects the minimum-norm allocation in the vanishing-regularization limit~\citep{li2026nonlinear}.

\subsection{Computation}
\label{sec:elc-computation}
We compute ELC by projected gradient descent on the convex payoff
objective over the efficiency hyperplane. Projection onto the efficiency hyperplane is
closed form and requires no linear optimization oracle.
Differentiating \cref{eq:smoothed-excess} gives
\begin{equation}
  \frac{\partial \slcv(x)}{\partial x_i}
  = -
  \frac{\sum_{S\in\cS,\, i\in S} \gw{S}{x}}{\sum_{S\in\cS} \gw{S}{x}}
  = -\Prob_{S\sim\gibbsdist}[\, i\in S \,],
  \label{eq:grad}
\end{equation}
where $\gibbsdist$ is the Gibbs distribution on $\cS$,
\begin{equation}
  \gibbsdist(S)
  := \frac{\gw{S}{x}}{\sum_{S'\in\cS} \gw{S'}{x}}.
  \label{eq:gibbs-dist}
\end{equation}
\Cref{eq:grad} shows that the gradient of the ELC objective is the negated membership probability under the Gibbs distribution $\gibbsdist$, which puts more mass on coalitions with larger excess.

Let $r(x)\in\R^n$ with $r_i(x)=-\Prob_{S\sim\gibbsdist}[i\in S]$.
The projected gradient update is
\begin{equation}
  x \leftarrow x - \alpha \bigl( r(x) - \bar r(x)\,\one \bigr),
  \label{eq:grad-update}
\end{equation}
where $\alpha>0$ is the step size, $\bar r(x):=\frac{1}{n}\sum_{i=1}^n r_i(x)$, and $\one$ is the all-ones vector. Subtracting $\bar r(x)\one$ keeps
$x(N):=\sum_{i=1}^n x_i$ unchanged over iterations. If we initialize, for example, at the uniform efficient payoff $x_0=(v(N)/n)\one$, every iterate remains efficient. Standard smooth-convex guarantees give objective-error rates
$O(1/T)$ after $T$ iterations with step size $4\eta/n$ and $O(1/T^2)$ with Nesterov
acceleration (\cref{prop:pgd-rate}). 
In practice, we approximate the ELC path by annealing, computing $\eta$-ELC at
successively smaller values of~$\eta$. 

%% file: sections/05_empirical_validation.tex
\input{tables/scaling_values}
\section{Computation and payoff selection}
\label{sec:selector}
\label{sec:experiments}
\label{sec:quality-check}
\label{sec:which-point}
\label{sec:computational-efficiency}

This section presents the empirical validation of ELC. Throughout the paper, unless otherwise stated, all the experiments were run on a Linux server with an AMD EPYC 9124 16-core CPU and an NVIDIA RTX PRO 6000 Blackwell Max-Q with the following settings.

\paragraph{Experimental setups.}
\label{par:experimental-setup}
\textbf{(ELC)} Our full-coalition experiments enumerate all $M=2^n-2$ nonempty proper
coalitions and use uniform ELC weights $w_S=1/M$.
We optimize ELC by projected gradient descent with Nesterov acceleration,
step size $\alpha=4\eta/n$, and initial payoff $x_0=(v(N)/n)\one$.
We use nine annealing stages from $\eta=1$ to $0.002$, warm-starting
each stage from the preceding payoff.  Each stage stops at a projected-gradient
infinity norm of $10^{-8}$ or after $10{,}000$ updates.
The full temperature schedule and solver details are in
\cref{app:small-game-protocol}.
\textbf{(Baselines)}
We compare ELC with LP, LP + min-norm, and LP + nucleolus, using
HiGHS for linear programs and SLSQP for minimum-norm selection.
The nucleolus solver follows the successive-LP procedure in
\cref{app:nucleolus}.

\begin{wraptable}{r}{0.40\textwidth}
    \caption{Least-core quality.
    Entries are mean $\pm$ standard error over five games.
    Rel.\ gap uses the coalition-value range.}
    \label{tab:quality-check}
    \centering
    \footnotesize
    \setlength{\tabcolsep}{2pt}
    \input{tables/quality_recovery}
\end{wraptable}
\paragraph{Least core value precision.}
Defining random games with $v(S)=|S|/n+\xi_S$, independent
$\xi_S\sim\mathcal N(0,0.3^2)$, and $v(N)=1$, ELC attains
mean hard-excess gaps $\etahat(x)-\lcv$ of order $10^{-3}$ against the full LP
(\cref{tab:quality-check}). 

\paragraph{ELC payoff in the limit.} We introduce UCI games from~\citet{yan2021} in addition to random games, and observe in \Cref{tab:selector-proximity} that the ELC payoff approaches the nucleolus in the limit of $\eta\to 0$.
As temperature decreases, the median distance to the nucleolus falls
monotonically in all three settings (\cref{fig:selector-path}); individual
paths need not be monotone.

\WFclear

\begin{table}[t]
\centering
\caption{ELC approximates the nucleolus at lower computational cost.
Distances are $d_{\rm nuc}=\|x_{\rm ELC}-x_{\rm nuc}\|_2$ and
$d_{\rm mn}=\|x_{\rm ELC}-x_{\rm mn}\|_2$; gap is
$\etahat(x_{\rm ELC})-\lcv$. Distances and gaps use medians across
20 games per random setting or ten UCI splits, after taking within-game
medians over three repetitions. Times average optimization runtimes,
excluding game construction and independent evaluation.
Further details are in \cref{app:speed-claims}.
$\dagger$ One failed min-norm reference is excluded.
$\ddagger$ House nucleolus time averages 15 completed repetitions,
excluding 15 timeouts under a 60-minute limit; distances use all ten splits.}
\label{tab:selector-proximity}
\label{tab:nucleolus-speed}
\small\setlength{\tabcolsep}{4pt}
\input{tables/selector_combined_main}
\end{table}

\paragraph{Runtime.}
\Cref{tab:selector-proximity} compares finite-annealing ELC with the
successive-LP nucleolus solver at their attained accuracies.
ELC takes less solver time in the reported settings.
\Cref{tab:gpu-speed} further shows that GPU acceleration benefits ELC
at larger coalition counts. Sampling reduces the cost of the initial LP,
but the subsequent min-norm selection can be slower and less reliable.
The attribution results in \cref{sec:generalized-freerider} show that
sampling particularly degrades distractor identification, while
useful-player identification remains comparatively strong.

%% file: tables/scaling_values.tex
\providecommand{\ScalingCpuSecondsTwenty}{15.129}
\providecommand{\ScalingGpuSecondsTwenty}{9.769}
\providecommand{\ScalingLpSecondsTwenty}{12.645}
\providecommand{\ScalingMinNormSecondsTwenty}{13.277}
\providecommand{\ScalingHalfLpSecondsTwenty}{7.134}
\providecommand{\ScalingTenthLpSecondsTwenty}{1.306}
\providecommand{\ScalingSpeedupTwenty}{1.55}
\providecommand{\ScalingHalfLpGapTwenty}{0.267}
\providecommand{\ScalingTenthLpGapTwenty}{0.271}
\providecommand{\ScalingCpuSecondsTwentyTwo}{145.965}
\providecommand{\ScalingGpuSecondsTwentyTwo}{46.525}
\providecommand{\ScalingLpSecondsTwentyTwo}{62.547}
\providecommand{\ScalingMinNormSecondsTwentyTwo}{65.457}
\providecommand{\ScalingHalfLpSecondsTwentyTwo}{44.586}
\providecommand{\ScalingTenthLpSecondsTwentyTwo}{7.322}
\providecommand{\ScalingSpeedupTwentyTwo}{3.14}
\providecommand{\ScalingHalfLpGapTwentyTwo}{0.138}
\providecommand{\ScalingTenthLpGapTwentyTwo}{0.27}
\providecommand{\ScalingCpuSecondsTwentyFour}{600.000$^{+}$}
\providecommand{\ScalingGpuSecondsTwentyFour}{214.134}
\providecommand{\ScalingLpSecondsTwentyFour}{317.264}
\providecommand{\ScalingMinNormSecondsTwentyFour}{330.111}
\providecommand{\ScalingHalfLpSecondsTwentyFour}{180.193}
\providecommand{\ScalingTenthLpSecondsTwentyFour}{25.628}
\providecommand{\ScalingSpeedupTwentyFour}{\underline{\hspace{1.5em}}}
\providecommand{\ScalingHalfLpGapTwentyFour}{0.133}
\providecommand{\ScalingTenthLpGapTwentyFour}{0.333}

%% file: tables/quality_recovery.tex
\begin{tabular}{rrrr}
\toprule
$n$ & $M$ & Gap ($10^{-3}$) & Rel. gap ($10^{-4}$) \\
\midrule
8 & 254 & $0.81\pm0.17$ & $4.00\pm0.81$ \\
10 & 1{,}022 & $0.96\pm0.30$ & $4.20\pm1.32$ \\
12 & 4{,}094 & $1.23\pm0.12$ & $4.96\pm0.50$ \\
14 & 16{,}382 & $1.34\pm0.14$ & $5.19\pm0.51$ \\
16 & 65{,}534 & $1.17\pm0.10$ & $4.17\pm0.35$ \\
18 & 262{,}142 & $1.64\pm0.15$ & $5.72\pm0.62$ \\
\bottomrule
\end{tabular}

%% file: tables/selector_combined_main.tex
\begin{tabular}{lrrr|rr}
\toprule
 & \multicolumn{3}{c|}{ELC (CPU)} & \multicolumn{2}{c}{Reference} \\
\cmidrule(lr){2-4}\cmidrule(lr){5-6}
Setting ($n$) & Time (s) & $d_{\rm nuc}$ ($10^{-3}$) & Gap ($10^{-3}$) & $d_{\rm mn}$ ($10^{-3}$) & LP + nucleolus (s) \\
\midrule
Random (6) & 0.253 & 1.97 & 0.59 & 98.28 & 1.117 \\
Random (7) & 0.190 & 1.66 & 0.79 & 24.38 & 1.065 \\
Random (8) & 0.384 & 2.84 & 0.28 & 81.36 & 1.369 \\
Medical (9) & 0.089 & 0.90 & 0.10 & 4.56 & 1.041 \\
Chemical (13) & 1.146 & 9.64 & 1.02 & 21.59$^{\dagger}$ & 45.321 \\
House (16) & 5.414 & 8.96 & 1.27 & 13.18 & 2150.567$^{\ddagger}$ \\
\bottomrule
\end{tabular}

%% file: sections/06_distractors.tex
\section{Identifying useful and harmful players}
\label{sec:identification}
\label{sec:generalized-freerider}

To test how context rarity and constraint sampling affect identification,
we generalize \cref{eq:simple-freerider}.
\paragraph{Contextual-distractor game.}
Let $U,D\subseteq N$ be disjoint sets of useful players and distractors;
the remaining players are neutral. Useful player $i$ contributes
$m_i>0$, and distractor $j$ contributes $a_j>0$.
A bonus $b>0$ becomes available when at least $\tau$ useful players are
present; each included distractor reduces it by $h_j>0$.
For a coalition $S\subseteq N$, define
\begin{equation}
 v(S)=\sum_{i\in S\cap U}m_i+\sum_{j\in S\cap D}a_j
 +\mathbf1[|S\cap U|\ge\tau]
       \left(b-\sum_{j\in S\cap D}h_j\right)_+,
 \label{eq:freerider-game}
\end{equation}
where $(z)_+:=\max\{z,0\}$.
Setting $\tau=|U|$, $a_j=a$, and $h_j=b$ recovers
\cref{eq:simple-freerider}. Adding distractor $j$ to a coalition with
at least $\tau$ useful players and no distractors reduces its value by
$h_j-a_j$ when $a_j<h_j\le b$.

\paragraph{Setup.}
We use $(n,|U|,|D|)=(10,6,2)$ and $(14,8,3)$, with neutral players
filling the remaining positions. For each of ten game seeds,
$\tau\in\{\lceil|U|/2\rceil,\lceil3|U|/4\rceil,|U|\}$;
$b=2|U|$ and $a_j=0.1|U|$. Useful effects vary across seeds.
Distractors have distinct, equally spaced $h_j/b$ values from $0.25$ to
$0.75$, permuted across players. Each game has five nested 10\%/50\% uniform
samples, shared across LP, LP + min-norm, and ELC. All methods are also
run on the full family. We evaluate average precision (AP): AP$^+$ ranks $x_i$ against useful-player membership,
and AP$^-$ ranks $-x_i$ against distractor membership. We also measure payoff signs and within-group strength ordering.

\paragraph{Results.}
\Cref{tab:strength} summarizes the results. First, with full coalitions, the AP$^+$ and AP$^-$ scores show that all methods identify all useful players correctly, but the Shapley value and the least core without a specified selector show limited performance in identifying distractors. Notably, ELC not only identifies the distractor players correctly, but also ranks them more accurately than the minimum-norm selector. In the 50\%/10\% sampling settings, AP$^-$ deteriorates across methods except LP, while AP$^+$ remains strong. This suggests that the sampling effect is more pronounced on distractor identification than on useful-player identification. Namely, while the least core with a selector is more advantageous than the Shapley value for identifying the distractors as \citet{nagarajan2026} pointed out, the sampling approach, a typical acceleration strategy in the literature~\citep{yan2021,gemp2024}, can reduce this advantage.
The loss is more pronounced at higher context thresholds and with fewer sampled coalitions (\cref{fig:sampling-n10,fig:sampling-n14}).

We observed that the larger drop in AP$^-$ than in AP$^+$ persists for ELC and
LP + min-norm when $|U|=|D|$, so unequal group sizes alone do not explain
this difference (see \cref{tab:strength-balanced}).

\begin{table}[t]
\centering\footnotesize\setlength{\tabcolsep}{3pt}
\caption{Full-coalition ELC identifies and orders both player groups;
sampling mainly reduces distractor identification.
AP$^+$/AP$^-$ measure useful/distractor identification.
Negative rate is the fraction of distractors receiving negative credit;
order accuracy is the fraction of distinct-strength pairs ranked correctly,
counting ties as incorrect. All scores range from 0 to 1; higher is better,
and bold marks the best mean within each budget, including ties.
Means average samples within each of 60 games, then across games.
Sampled methods share coalition observations.
Further evaluation details are in \cref{app:sampling-protocol}.}
\label{tab:strength}
\input{tables/strength_main}
\end{table}

%% file: tables/strength_main.tex
\begin{tabular}{llrrrrr}
\toprule
Constraints & Method & AP$^+$ & AP$^-$ & Distractor & Useful & Harm \\
 & & & & negative rate & order accuracy & order accuracy \\
\midrule
Full & Shapley & \textbf{1.000} & 0.794 & 0.639 & \textbf{1.000} & \textbf{1.000} \\
 & LP & \textbf{1.000} & 0.826 & 0.667 & 0.918 & 0.617 \\
 & LP + min-norm & \textbf{1.000} & \textbf{1.000} & \textbf{1.000} & 0.667 & 0.333 \\
 & ELC & \textbf{1.000} & \textbf{1.000} & \textbf{1.000} & \textbf{1.000} & \textbf{1.000} \\
\midrule
50\% sample & Shapley & 0.920 & 0.689 & 0.635 & 0.563 & 0.740 \\
 & LP & 0.998 & 0.827 & 0.693 & \textbf{0.922} & 0.636 \\
 & LP + min-norm & 0.999 & \textbf{0.967} & \textbf{0.962} & 0.667 & 0.432 \\
 & ELC & \textbf{1.000} & \textbf{0.967} & \textbf{0.962} & 0.899 & \textbf{0.959} \\
\midrule
10\% sample & Shapley & 0.883 & 0.546 & 0.474 & 0.545 & 0.689 \\
 & LP & 0.980 & 0.824 & 0.770 & \textbf{0.760} & 0.647 \\
 & LP + min-norm & \textbf{0.983} & \textbf{0.879} & \textbf{0.877} & 0.725 & 0.506 \\
 & ELC & \textbf{0.983} & 0.860 & 0.864 & 0.720 & \textbf{0.842} \\
\bottomrule
\end{tabular}

%% file: sections/07_applications.tex
\section{Applications to training-data valuation}
\label{sec:applications}
\label{sec:data-valuation}

We now compare the Shapley value and the least core methods in downstream tasks. We adopt the training-data valuation task from \citet{yan2021}. Further results and experiments are provided in \cref{sec:yan-protocol,app:marker}. 

\paragraph{Metrics.} Training points were removed in increments of 5\%, refitting the predictor
after removal. Removing the highest-ranked points (Best) tests whether
important examples receive high credit; lower retained accuracy is
better. Removing the lowest-ranked points (Worst) tests whether useful
examples are retained; higher is better. We integrate each
accuracy curve over 0--35\% removal and divide by 0.35 to obtain the
normalized area under the curve (AUC). Differences are
paired within repetition; uncertainty is computed over repetitions,
not over points on the removal curve.

\input{tables/yan_data_paired_auc}

\paragraph{Results.}
\Cref{tab:yan-data-paired-auc} shows that ELC improves mean removal AUC over the minimum-norm least core and Shapley value in nearly all comparisons. ELC-minus-baseline AUC differences are negative for best-first removal and positive for worst-first removal, except for Dogfish best-first removal at 10K queries versus the minimum-norm least core. 

%% file: tables/yan_data_paired_auc.tex
\begin{table*}[t]
\centering
\caption{Paired differences in normalized removal AUC (percentage points): ELC minus each baseline. Best removes the highest-ranked training points; Worst removes the lowest-ranked. Negative differences favor ELC for Best; positive differences favor ELC for Worst. CI denotes the 95\% bootstrap confidence interval for the mean difference. \(N_{\mathrm{query}}\) is the coalition-query budget.}
\label{tab:yan-data-paired-auc}
\small
\setlength{\tabcolsep}{5pt}
\begin{tabular}{rlrcr c}
\toprule
& & \multicolumn{2}{c}{ELC $-$ min-norm LC} & \multicolumn{2}{c}{ELC $-$ Shapley} \\
\cmidrule(lr){3-4}\cmidrule(lr){5-6}
$N_{\mathrm{query}}$ & Removal & Mean & 95\% CI & Mean & 95\% CI \\
\midrule
\multicolumn{6}{l}{\textit{Gaussian (20 repetitions per budget)}} \\
5K & Best $\downarrow$ & \(-0.463\) & \([-0.613,\,-0.306]\) & \(-3.312\) & \([-3.889,\,-2.769]\) \\
5K & Worst $\uparrow$ & \(+0.118\) & \([+0.037,\,+0.197]\) & \(+0.076\) & \([-0.206,\,+0.330]\) \\
\addlinespace[2pt]
10K & Best $\downarrow$ & \(-0.439\) & \([-0.582,\,-0.305]\) & \(-4.646\) & \([-5.330,\,-4.023]\) \\
10K & Worst $\uparrow$ & \(+0.121\) & \([+0.046,\,+0.193]\) & \(+0.468\) & \([+0.097,\,+0.831]\) \\
\addlinespace[2pt]
25K & Best $\downarrow$ & \(-0.355\) & \([-0.489,\,-0.217]\) & \(-4.292\) & \([-4.812,\,-3.779]\) \\
25K & Worst $\uparrow$ & \(+0.210\) & \([+0.146,\,+0.272]\) & \(+0.728\) & \([+0.446,\,+1.005]\) \\
\addlinespace[2pt]
50K & Best $\downarrow$ & \(-0.495\) & \([-0.673,\,-0.327]\) & \(-3.185\) & \([-3.609,\,-2.715]\) \\
50K & Worst $\uparrow$ & \(+0.165\) & \([+0.083,\,+0.255]\) & \(+0.382\) & \([+0.167,\,+0.608]\) \\
\midrule
\multicolumn{6}{l}{\textit{Dog--fish (5 repetitions per budget)}} \\
5K & Best $\downarrow$ & \(-0.074\) & \([-0.143,\,-0.005]\) & \(-0.402\) & \([-0.581,\,-0.217]\) \\
5K & Worst $\uparrow$ & \(+0.012\) & \([-0.002,\,+0.024]\) & \(+0.005\) & \([-0.057,\,+0.086]\) \\
\addlinespace[2pt]
10K & Best $\downarrow$ & \(+0.005\) & \([-0.036,\,+0.045]\) & \(-0.360\) & \([-0.500,\,-0.219]\) \\
10K & Worst $\uparrow$ & \(+0.033\) & \([+0.000,\,+0.071]\) & \(+0.043\) & \([+0.029,\,+0.064]\) \\
\addlinespace[2pt]
25K & Best $\downarrow$ & \(-0.095\) & \([-0.195,\,-0.007]\) & \(-0.593\) & \([-0.671,\,-0.512]\) \\
25K & Worst $\uparrow$ & \(+0.024\) & \([+0.005,\,+0.048]\) & \(+0.088\) & \([+0.019,\,+0.164]\) \\
\addlinespace[2pt]
50K & Best $\downarrow$ & \(-0.052\) & \([-0.086,\,-0.019]\) & \(-0.345\) & \([-0.414,\,-0.283]\) \\
50K & Worst $\uparrow$ & \(+0.036\) & \([-0.024,\,+0.088]\) & \(+0.326\) & \([+0.200,\,+0.486]\) \\
\bottomrule
\end{tabular}
\end{table*}

%% file: sections/08_conclusion.tex
\section{Conclusion}
\label{sec:conclusion}

We studied least-core attribution for contextual distractors, which cause substantial harm in rare, high-value contexts, from a computational perspective. Our experiments show that selector choice is crucial for distractor retrieval and that coalition sampling degrades distractor identification more severely than useful-player identification. We introduced entropic least core (ELC), a smooth log-sum-exp approximation whose minimizer converges to the nucleolus as the temperature $\eta\to0$. In our experiments, ELC approximates the nucleolus in less solver time than the successive-LP procedure. ELC and the minimum-norm selector both achieve high average precision for distractor identification in the tested games; ELC further ranks distractors more accurately by harm.

\paragraph{Limitations and future work.} Least-core attribution for model explanation is still developing, with two open questions. First, the practical advantages of the least core over the Shapley value remain incompletely understood. Existing experiments span feature attribution, data valuation, and label correction, but these are general tasks. Identifying tasks that benefit specifically from contextual-distractor retrieval remains an important direction. Second, full coalition enumeration has exponential cost in the number of players. Recent work proposes approximation methods, but the method of \citet{gemp2024} restricts payoffs to be nonnegative, and coalition sampling as in \citet{yan2021} can degrade distractor identification. Designing a sampling method, or other approximation techniques, that preserves this ability while providing approximation guarantees is an open direction for scalable least-core computation.

%% file: sections/A_appendix.tex
\section{Additional related work}
\label{app:additional-related}

We expand the attribution literature discussed in \cref{sec:related}.

The Shapley value~\citep{shapley1953} assigns credit by averaging a
player's marginal contribution over the orders in which players can join
a coalition. It supports local feature explanations in
\citet{strumbelj2010} and SHAP~\citep{lundberg2017}, global feature
importance in SAGE~\citep{covert2020}, and the valuation of training
examples in Data Shapley~\citep{ghorbani2019}.
Related game-theoretic approaches also study groups of players and their
interactions, including in computer-vision explanations~\citep{grabisch1999an,sumiyasu2022,sumiyasu2024,pred-diff,chang2025}.
The choice of players and coalition values determines what an attribution
explains. Even for feature attribution, how absent features are represented affects
the explanation~\citep{sundararajan2020}; \citet{aas2021} specifically
address dependencies between features.

Computational work reduces the cost of Shapley attribution through
permutation sampling~\citep{castro2009polynomial,mitchell2022} or learned estimators such as
FastSHAP~\citep{jethani2022}.
These methods approximate an averaging-based allocation rule.
Semivalues generalize this rule by changing the weighting of coalition
sizes~\citep{dubey1981,kwon2022beta,wang2023data}.
Such averaging is natural when the target is average contribution, but
it can obscure harm concentrated in important contexts, as formalized by
\citet{nagarajan2026}.
The least core instead minimizes the largest coalition shortfall.
This distinction concerns which allocation is sought, separately from
how accurately it is computed: our comparisons examine both the allocation
rule and the effect of coalition sampling on the resulting explanation.

\section{Least-core selectors and entropic regularization}
\label{app:selector-theory}

This section supplies the selector definitions and proofs used in
\cref{sec:background,sec:method}. We describe the nucleolus, establish the
entropic variational representation, and apply \citet{alvarez2000} to
identify the ELC limit.

\subsection{Nucleolus as a lexicographic least-core selector}
\label{app:nucleolus}

Recall that $N=\{1,\ldots,n\}$ is the player set and
$\mathcal S=2^N\setminus\{\emptyset,N\}$ is the family of nonempty
proper coalitions, with $M=|\mathcal S|=2^n-2$.
The nucleolus refines the least core by minimizing coalition excesses
lexicographically. For an efficient payoff $x$, recall the excess of a
coalition $S$:
\[
  e_S(x) := v(S)-x(S).
\]
Large $e_S(x)$ means that coalition $S$ is underpaid by $x$.
Fix an enumeration $S_1,\ldots,S_M$ of $\mathcal S$.
For each payoff $x$, let $(\sigma_1(x),\ldots,\sigma_M(x))$ be a
permutation of $\{1,\ldots,M\}$ such that
\[
  e_{S_{\sigma_1(x)}}(x)\ge e_{S_{\sigma_2(x)}}(x)\ge\cdots\ge e_{S_{\sigma_M(x)}}(x).
\]

Here $\sigma_k(x)$ is the index of the coalition with the $k$th-largest
excess at $x$.  The coalitions $S_1,\ldots,S_M$ are fixed, while the
sorting permutation depends on $x$.  Ties may be ordered arbitrarily, since
they do not change the sorted vector.
The nucleolus is the efficient payoff whose sorted deficit vector is
lexicographically minimal:
\[
  x_{\mathrm{nuc}}
  := \argmin_{x:\,x(N)=v(N)}^{\mathrm{lex}}
     \bigl(e_{S_{\sigma_1(x)}}(x),e_{S_{\sigma_2(x)}}(x),\ldots,e_{S_{\sigma_M(x)}}(x)\bigr).
\]
The first coordinate of this lexicographic program is the least-core
value $\lcv$.
Thus the nucleolus lies in the least core, and can be read as a
selector for $\mathrm{LC}(\lcv)$.
It is stricter than picking an arbitrary least-core point: after
minimizing the worst deficit, it minimizes the next-worst remaining
deficit, then the next one, until a single payoff remains.

The finite procedure is a sequence of linear programs.
At level $k$, some coalitions have already been fixed at previous
deficit levels $\theta_1,\ldots,\theta_{k-1}$.
Among all payoffs satisfying those equalities, solve
\[
\begin{aligned}
  \min_{x,\theta_k}\quad & \theta_k \\
  \mathrm{s.t.}\quad
    & x(N)=v(N),\\
    & e_S(x)=\theta_j
      \quad\text{for all coalitions fixed at level }j<k,\\
    & e_S(x)\le \theta_k
      \quad\text{for every unfixed coalition }S.
\end{aligned}
\tag{NUC-$k$}
\]
A coalition is fixed only if its excess equals the current optimum
throughout the optimal face. Tightness at one LP solution is insufficient:
a constraint may be tight at the returned vertex but slack at another
optimal payoff. The following procedure tests this by maximizing each
coalition's slack over the face.

Let $F$ contain the fixed coalitions and $\theta(S)$ their fixed excess levels:
\[
\begin{array}{ll}
\textbf{Input:} & \{(S,v(S)):S\in\mathcal S\},\ v(N).\\
\textbf{Initialize:} & F\leftarrow\emptyset,\ k\leftarrow1.\\
\textbf{Repeat:} &
  \text{solve (NUC-$k$), obtaining its optimal value }\theta_k.\\
& T\leftarrow\emptyset.\\
& \text{For each }S\in\mathcal S\setminus F:\\
& \quad \text{maximize } \theta_k-e_S(x)
       \text{ over the optimal face of (NUC-$k$).}\\
& \quad \text{If the maximum is }0,\text{ add }S\text{ to }T.\\
& \text{Set }F\leftarrow F\cup T\text{ and }\theta(S)\leftarrow\theta_k
   \text{ for }S\in T.\\
& k\leftarrow k+1.\\
\textbf{Until:} &
  x \text{ is uniquely determined.}\\
\textbf{Output:} & x.
\end{array}
\]
\subsection{Entropic representation}
\label{app:entropy-representation}

Recall that, for $\eta>0$, the smoothed maximum excess is defined by
\[
  \slcv(x):=\eta\log\sum_{S\in\cS}
  \exp\!\left(\frac{e_S(x)}{\eta}\right).
\]
The following identity gives the entropic interpretation of ELC and
its approximation bound for the maximum excess.

\begin{proposition}[Entropic variational representation]
\label{prop:entropic-variational}
Let $\Delta$ be the probability simplex on $\cS$, let $w\in\Delta$
have full support, and let $a\in\R^M$ and $\eta>0$. Then
\[
  \max_{\lambda\in\Delta}
  \left\{\sum_S\lambda_S a_S-\eta\KL(\lambda\|w)\right\}
  =\eta\log\sum_S w_S e^{a_S/\eta}.
\]
The unique maximizer is
$\lambda^*_S=w_S e^{a_S/\eta}/\sum_T w_T e^{a_T/\eta}$.
For $a_S=e_S(x)$ and $w_S=1/M$, the value equals
$\slcv(x)-\eta\log M$, and
\[
  \etahat(x)\le\slcv(x)\le\etahat(x)+\eta\log M,
  \qquad \etahat(x):=\max_{S\in\cS}e_S(x).
\]
\end{proposition}

\begin{proof}
Set $Z=\sum_S w_S e^{a_S/\eta}$. The definition of $\lambda^*$ gives
\[
  a_S=\eta\bigl(\log\lambda^*_S-\log w_S+\log Z\bigr).
\]
For any $\lambda\in\Delta$, substituting this expression for $a_S$
into the objective and using $\sum_S\lambda_S=1$ yields
\[
  \sum_S\lambda_S a_S-\eta\KL(\lambda\|w)
  =\eta\log Z-\eta\KL(\lambda\|\lambda^*).
\]
The KL divergence is nonnegative and vanishes only at
$\lambda=\lambda^*$, proving the identity and uniqueness.
For uniform $w$, evaluating the variational objective at a point mass
on a maximizing coalition gives the lower bound
$\etahat(x)-\eta\log M$; its upper bound is $\etahat(x)$ because
$\sum_S\lambda_S e_S(x)\le\etahat(x)$ and $\KL(\lambda\|w)\ge0$.
Adding $\eta\log M$ proves the stated bounds.
\end{proof}

\subsection{ELC as an exponential-penalty trajectory}
\label{app:zero-temperature-selector}

We prove \cref{thm:elc-properties}, including convergence to the nucleolus,
by expressing ELC as an exponential-penalty trajectory. We first establish
the compactness needed for existence and for the penalty theorem.

\begin{lemma}[Compactness of excess sublevel sets]
\label{lem:excess-sublevels}
For $n\ge2$ and $\cS=2^N\setminus\{\emptyset,N\}$, let
$\mathcal X=\{x\in\R^n:x(N)=v(N)\}$. For every $B\in\R$, the set
\[
  K_B:=\{x\in\mathcal X:\etahat(x)\le B\}
\]
is compact, where $\etahat(x):=\max_{S\in\cS}e_S(x)$.
In particular, the least core is nonempty and compact.
\end{lemma}

\begin{proof}
We show that $K_B$ is bounded and closed in $\R^n$, which implies
compactness by the Heine--Borel theorem.
Membership in $K_B$ means that $v(S)-x(S)\le B$ for every
$S\in\cS$. Taking $S=\{i\}$ gives
\[
  v(\{i\})-x_i\le B,
  \qquad\text{so}\qquad x_i\ge v(\{i\})-B.
\]
Taking $S=N\setminus\{i\}$ and using efficiency, $x(N)=v(N)$, we have
$x(N\setminus\{i\})=x(N)-x_i=v(N)-x_i$, which gives
\[
  v(N\setminus\{i\})-v(N)+x_i\le B,
  \qquad\text{so}\qquad
  x_i\le B+v(N)-v(N\setminus\{i\}).
\]
Thus every coordinate has a finite lower and upper bound, and $K_B$
is bounded. It is also closed, since $\etahat$ is continuous.

We next show that the least core is nonempty and compact.
Choose any $\bar x\in\mathcal X$ and fix $B=\etahat(\bar x)$.
Then $K_B$ is nonempty and compact, so $\etahat$ attains a minimum
there. Outside $K_B$, $\etahat(x)>B$, so this is also a minimum on
$\mathcal X$. The least core is therefore nonempty; it is a closed subset of this
compact $K_B$, and hence is compact.
\end{proof}

We use \cref{prop:entropic-variational} to prove existence and
approximation quality, then apply the exponential-penalty theorem of
\citet{alvarez2000} to identify the limiting selector.

\elcproperties*

\begin{proof}[Proof of \cref{thm:elc-properties}]
\textbf{Existence and uniqueness (statement 1).}
Let $\mathcal X:=\{x\in\R^n:x(N)=v(N)\}$ be the efficient-payoff
hyperplane, and choose any $\bar x\in\mathcal X$. Recall that
$e_S(x):=v(S)-x(S)$. The ELC objective is
\[
  \varepsilon_\eta(x):=\eta\log\sum_{S\in\cS}
      \exp\!\left(\frac{e_S(x)}{\eta}\right).
\]
Consider
\[
  L_\eta:=\{x\in\mathcal X:\varepsilon_\eta(x)
      \le\varepsilon_\eta(\bar x)\}.
\]
This set contains $\bar x$ and is closed by continuity.

Write $\hat\varepsilon(x):=\max_{S\in\cS}e_S(x)$. By the
log-sum-exp bound in \cref{prop:entropic-variational}, it satisfies
$\hat\varepsilon(x)\le\varepsilon_\eta(x)$. Hence
$L_\eta\subseteq K_{\varepsilon_\eta(\bar x)}$, where
$K_B:=\{x\in\mathcal X:\hat\varepsilon(x)\le B\}$ is compact by
\cref{lem:excess-sublevels}.
The continuous ELC objective attains its minimum on the compact set
$L_\eta$, and points outside this set cannot improve that minimum.
By \cref{lem:smoothness}, $\varepsilon_\eta$ is strictly convex on $\mathcal X$,
so its minimizer $x_\eta$ is unique.

\textbf{Approximation quality (statement 2).}
Write $\varepsilon^*:=\min_{x\in\mathcal X}\hat\varepsilon(x)$ for the
least-core value, and choose a payoff $x^0$ attaining this minimum.
Such a payoff exists by \cref{lem:excess-sublevels}. Let $M:=|\cS|$.
Then \cref{prop:entropic-variational} and
optimality of $x_\eta$ give
\[
  \varepsilon^*\le\hat\varepsilon(x_\eta)
  \le\varepsilon_\eta(x_\eta)
  \le\varepsilon_\eta(x^0)\le\varepsilon^*+\eta\log M.
\]
The first inequality holds because $\varepsilon^*$ is the minimum of
$\hat\varepsilon$ on $\mathcal X$. The second is the lower
log-sum-exp bound. The third follows because $x_\eta$ minimizes
$\varepsilon_\eta$, and the last applies the upper log-sum-exp bound
to $x^0$, for which $\hat\varepsilon(x^0)=\varepsilon^*$.
This proves the hard-gap bound and convergence of the smooth value.

\textbf{Identifying the selector (statement 3).}
We use the following affine case of \citet[Theorem~3.1]{alvarez2000}.
Let $f_0,f_1,\ldots,f_m:\R^d\to\R$ be affine functions, and let $Z^*$
be the set of solutions to
\[
  \min_{z\in\R^d} f_0(z)
  \quad\text{s.t.}\quad f_i(z)\leq0\quad(i=1,\ldots,m).
\]
If $Z^*$ is nonempty and compact, then, for every $\eta>0$, the function
\[
  f_0(z)+\eta\sum_{i=1}^m\exp\bigl(f_i(z)/\eta\bigr)
\]
has a unique minimizer.  As $\eta$ tends to zero through positive values,
this minimizer converges to the point in $Z^*$ that lexicographically
minimizes the constraint values $f_1(z),\ldots,f_m(z)$ sorted from largest
to smallest.  Alvarez calls this point the \emph{absolute minimizer}.

To apply this result, introduce a scalar $t$ that upper-bounds every
excess and write the least-core problem as
\[
  P:\quad \min_{x\in\mathcal X,\,t\in\R} t
  \quad\text{s.t.}\quad e_S(x)-t\le0\quad(S\in\cS).
\]
We first check the affine-function assumption.  Eliminate efficiency by writing
$x=\bar x+Qy$, where $y\in\R^{n-1}$ and the columns of
$Q\in\R^{n\times(n-1)}$ form a basis for
$\{u\in\R^n:\sum_i u_i=0\}$.
The functions $f_0(y,t)=t$ and
$f_S(y,t)=e_S(\bar x+Qy)-t$, one for each $S\in\cS$, are affine
in the joint variable $(y,t)$.

Next, the optimal set of $P$ is
$S(P)=\{(x,\varepsilon^*):x\in\mathcal X,
\hat\varepsilon(x)=\varepsilon^*\}$, which is nonempty and compact
by \cref{lem:excess-sublevels}.  The change of coordinates is a continuous
bijection with continuous inverse, so the optimal set in $(y,t)$ is also
nonempty and compact.  Both assumptions of the stated affine case are
therefore satisfied.

It remains to identify ELC with the exponential penalty in the theorem.
For $P$, the penalized problem is
\begin{equation}\label{eq:exponential-penalty}
  \min_{x\in\mathcal X,\,t\in\R}\Phi_\eta(x,t),
  \qquad
  \Phi_\eta(x,t):=t+\eta\sum_{S\in\cS}
       \exp\bigl((e_S(x)-t)/\eta\bigr).
\end{equation}
For fixed $x$, this function is strictly convex in $t$, and its
stationarity condition is $\sum_{S\in\cS}\exp((e_S(x)-t)/\eta)=1$.
Therefore its minimizer is $t=\varepsilon_\eta(x)$, and
$\min_t\Phi_\eta(x,t)=\varepsilon_\eta(x)+\eta$.

The joint minimizer is consequently
$(x_\eta,\varepsilon_\eta(x_\eta))$.  The theorem now gives convergence
of this pair to the absolute minimizer on $S(P)$ as $\eta$ tends to zero
through positive values.

Finally, on $S(P)$ the constraint values are $e_S(x)-\varepsilon^*$.
Subtracting the same constant $\varepsilon^*$ from every excess leaves
their lexicographic order unchanged.  Restricting to the least core has
already minimized the largest excess, so lexicographically minimizing
the remaining excesses gives the nucleolus, as defined in
\cref{app:nucleolus}.  The absolute minimizer's payoff component is thus
the nucleolus $x_{\mathrm{nuc}}$.
Hence $x_\eta\to x_{\mathrm{nuc}}$.
\end{proof}

\paragraph{Relation to \citet{abernethy2024}.}
\leavevmode %
\citet{abernethy2024} also study lexicographic optimization through
exponential-loss minimization. Our convergence result for the exact ELC
path, however, follows from the earlier exponential-penalty theorem of
\citet{alvarez2000}.
\citet{abernethy2024} additionally establish convergence guarantees under
lexicographic stability, including for approximate minimizers with
controlled errors, and derive convergence-rate bounds; we do not use
these results.
For optimization, they analyze a multiplicative-weights algorithm that,
with their choice of learning rates, is equivalent to Frank--Wolfe,
whereas we use projected gradient descent directly in payoff space
under the efficiency constraint
(\cref{sec:elc-computation}).

\section{Smoothness and optimization rates}
\label{app:proofs}

This section establishes the curvature and convergence bounds used by
the projected-gradient method in \cref{sec:method}.

\subsection{Smoothness of the ELC objective}

We bound the Hessian to justify the step size $4\eta/n$ and show strict
convexity on the efficiency hyperplane, which ensures uniqueness.

\begin{lemma}[Smoothness of the ELC objective]
\label{lem:smoothness}
The function $\slcv$ is convex and $\beta$-smooth with
$\beta=n/(4\eta)$, meaning that its gradient is $\beta$-Lipschitz
in the Euclidean norm.
On the efficiency hyperplane it is strictly convex.
\end{lemma}

\begin{proof}
The gradient and Hessian of $\slcv$ are
\begin{align}
  \nabla_{x} \slcv(x)
  &= - \E_{S \sim \gibbsdist}[\ind{S}],
  \\
  \nabla^2_{x} \slcv(x)
  &= \frac{1}{\eta}
    \Cov_{S \sim \gibbsdist}\bigl(\ind{S}\bigr)
  \succeq 0,
\end{align}
where $(\ind{S})_i=\mathbf1[i\in S]$ is the coalition-membership vector,
$\E$ denotes expectation, $\Cov$ covariance, and $\succeq0$ positive
semidefiniteness.
Convexity follows.

Each coordinate of $\ind{S}$ is Bernoulli: writing $\Var$ for variance and
$m_i(x):=\Prob_{S\sim\gibbsdist}[i\in S]$,
\[
  \Var_{S\sim\gibbsdist}\bigl((\ind{S})_i\bigr)
  = m_i(x)\bigl(1-m_i(x)\bigr)
  \le \tfrac{1}{4}.
\]
Writing $\|\cdot\|_{\rm op}$ for the Euclidean operator norm and
$\mathrm{tr}$ for the trace, the positive semidefinite Hessian's operator norm equals
its largest eigenvalue and is at most its trace:
\begin{align}
  \norm{\nabla^2_{x} \slcv(x)}_{\mathrm{op}}
  &= \frac{1}{\eta}\,
    \norm{\Cov_{S\sim\gibbsdist}(\ind{S})}_{\mathrm{op}}
  \le \frac{1}{\eta}\,
    \mathrm{tr}\bigl(\Cov_{S\sim\gibbsdist}(\ind{S})\bigr)
  \le \frac{n}{4\eta}.
\end{align}
Hence $\nabla\slcv$ is $\beta$-Lipschitz with $\beta=n/(4\eta)$.

For strict convexity on the efficiency hyperplane, take a nonzero
$d\in\R^n$ with $\one^{\top}d=0$.
The singleton coalitions have positive Gibbs probability and give
$d^{\top}\ind{\{i\}}=d_i$. These values cannot all be equal: together
with $\sum_i d_i=0$, equality would imply $d=0$. Hence
$d^{\top}\ind{S}$ has positive variance.
Thus $d^{\top}\Cov(\ind{S})d>0$, so the Hessian is positive definite on
the tangent space $\{d\in\R^n:\one^\top d=0\}$.
\end{proof}

\subsection{Projected-gradient convergence rates}

At fixed temperature, the smoothness bound gives the following
objective-error rates for ordinary and accelerated projected gradient
descent. Both methods preserve efficiency.

\begin{proposition}[Projected-gradient ELC]
  \label{prop:pgd-rate}
  Let $x_T$ be the output after $T\ge 1$ iterations of
  \cref{eq:grad-update} initialized at $x_0$ with $x_0(N)=v(N)$.
  Then the following hold:
  \begin{enumerate}
    \item The update preserves efficiency; hence, $x_T(N)=v(N)$.
    \item With step size $\alpha=1/\beta=4\eta/n$, if $x^\star$
          is the unique minimizer of \cref{eq:elc}, then
          \[
            \slcv(x_T)-\slcv(x^{\star})
            \le \frac{n}{8\eta T}\,\norm{x_0-x^{\star}}_2^2.
          \]
    \item Applying Nesterov acceleration to the same projected-gradient
          problem also preserves efficiency and improves the standard
          smooth-convex rate to
          \[
            \slcv(x_T)-\slcv(x^{\star})
            \le \frac{n}{2\eta (T+1)^2}\,\norm{x_0-x^{\star}}_2^2.
          \]
  \end{enumerate}
\end{proposition}

\begin{proof}[Proof of \cref{prop:pgd-rate}]
The update direction in \cref{eq:grad-update} has zero sum because
\begin{equation}
  \sum_{i=1}^n \bigl(r_i(x)-\bar r(x)\bigr)=0.
\end{equation}
Thus the update preserves $x(N)$, and an efficient initialization
remains efficient after any number of iterations.

By \cref{lem:smoothness}, $\slcv$ is $\beta$-smooth on the efficiency
hyperplane with $\beta=n/(4\eta)$.
Projected gradient descent with step size $\alpha=1/\beta$ over an
affine subspace therefore satisfies the standard smooth convex rate
\citep{nesterov2018}
\begin{equation}
  \slcv(x_T)-\slcv(x^\star)
  \le \frac{\beta}{2T}\norm{x_0-x^\star}_2^2
  = \frac{n}{8\eta T}\norm{x_0-x^\star}_2^2.
\end{equation}
Nesterov acceleration over the same affine subspace preserves
efficiency because it only takes affine combinations of efficient
payoff vectors and projected gradient steps with zero-sum directions.
The standard accelerated smooth convex guarantee gives
\citep{nesterov2018}
\begin{equation}
  \slcv(x_T)-\slcv(x^\star)
  \le \frac{2\beta}{(T+1)^2}\norm{x_0-x^\star}_2^2
  = \frac{n}{2\eta (T+1)^2}\norm{x_0-x^\star}_2^2.
\end{equation}
\end{proof}

\begin{corollary}[Maximum-excess accuracy]
\label{cor:maximum-excess-accuracy}
Let $n\ge2$, $\cS=2^N\setminus\{\emptyset,N\}$, and $M=|\cS|$.
At fixed temperature $\eta>0$, let $x_T$ be the accelerated iterate of
\cref{prop:pgd-rate} from an efficient initial payoff $x_0$, and set
$R_\eta=\norm{x_0-x_\eta}_2$. For $T\ge1$,
\[
  0\le\etahat(x_T)-\lcv
  \le \eta\log M+\frac{nR_\eta^2}{2\eta(T+1)^2}.
\]
In particular, for any $\delta>0$, the choices
\[
  \eta=\frac{\delta}{2\log M},\qquad
  T+1\ge\frac{R_\eta\sqrt{2n\log M}}{\delta}
\]
ensure $\etahat(x_T)-\lcv\le\delta$.
\end{corollary}

\begin{proof}
Let $x^{\rm LC}$ be a least-core payoff, which exists by
\cref{lem:excess-sublevels}. The log-sum-exp bounds and optimality of
$x_\eta$ give
\[
  \etahat(x_T)\le\slcv(x_T),\qquad
  \slcv(x_\eta)\le\slcv(x^{\rm LC})\le\lcv+\eta\log M.
\]
Combining these inequalities with the accelerated rate in
\cref{prop:pgd-rate} proves the upper bound. Efficiency of $x_T$
gives the lower bound. The stated choices make each upper-bound term
at most $\delta/2$.
\end{proof}

Since $M=2^n-2$, this gives a sufficient iteration budget of
$O(1+nR_\eta/\delta)$. The distance factor $R_\eta$ may itself depend
on dimension and temperature; the bound is not a necessary iteration
count. It controls the maximum-excess objective error, not payoff
distance to the nucleolus, and applies at a fixed temperature rather
than bounding the total cost of the annealing schedule.

\section{Experimental protocol details}
\label{app:experiment-protocols}
\label{app:small-game-protocol}

This section supplements the quality, selector, and player-identification
experiments in \cref{sec:selector,sec:identification} with solver settings
and a check of sensitivity to initialization.

ELC starts at $x_0=(v(N)/n)\one$ and uses Nesterov acceleration with
step size $4\eta/n$. The nine temperatures are
$\eta\in\{1,0.5,0.2,0.1,0.05,0.02,0.01,0.005,0.002\}$.
Each stage starts from the preceding payoff, with momentum reset, and
stops at a projected-gradient infinity norm of $10^{-8}$ or
$10{,}000$ iterations. An iteration-limited payoff initializes the next
stage.

\paragraph{Numerical stability.}
We evaluate the objective and Gibbs weights using stabilized log-sum-exp
and softmax routines. For $z_S=e_S(x)/\eta$ and $m=\max_S z_S$, these compute
\[
  \log\sum_S e^{z_S}=m+\log\sum_S e^{z_S-m},
  \qquad
  p_S=\frac{e^{z_S-m}}{\sum_T e^{z_T-m}}.
\]
The exponentials therefore have nonpositive arguments, avoiding overflow
from directly exponentiating large scaled excesses. The implementation
subtracts the constant $\eta\log M$ from the objective in
\cref{eq:smoothed-excess}; this leaves its gradient and minimizer unchanged.

The initial least-core LP in runtime comparisons uses HiGHS interior
point with crossover and primal and dual feasibility tolerances of
$10^{-9}$. Minimum-norm selection uses SLSQP, with constraint generation
for larger coalition families; its numerical relaxation is specified in
\cref{app:full-constraint-timing}.

Timing repetitions randomize method order within each game. Warm-up
computations are excluded, CUDA is synchronized at timing boundaries,
and numerical-library threads are limited to one per worker.
Standard errors use the sample standard deviation divided by the square
root of the number of independent games or splits.

\paragraph{Sensitivity to initialization.}
\label{app:initialization-diagnostic}
For each of 60 random games, we run the same annealing path from five
initializations: the uniform payoff, LP + min-norm, a TMC Shapley
estimate, and two random efficient payoffs. Endpoints agree to about
$10^{-12}$ within each game (\cref{tab:initialization}), indicating
negligible initialization sensitivity under this protocol. The table
also records how often the iteration limit is reached.

\input{tables/initialization}

\section{Annealing trajectories}
\label{app:selector-supplement}

We examine how the computed payoff approaches the nucleolus during
annealing in the random and UCI games of \cref{sec:selector}.
Median distance decreases at every measured temperature transition
in all three settings (\cref{fig:selector-path}). Individual paths
follow the same trend except in four of 60 random games, where the
largest upward step is $0.00279$. All displayed medical and chemical
paths decrease monotonically up to a $10^{-10}$ numerical tolerance.

\begin{figure}[!tb]
\centering
\includegraphics[width=.52\linewidth]{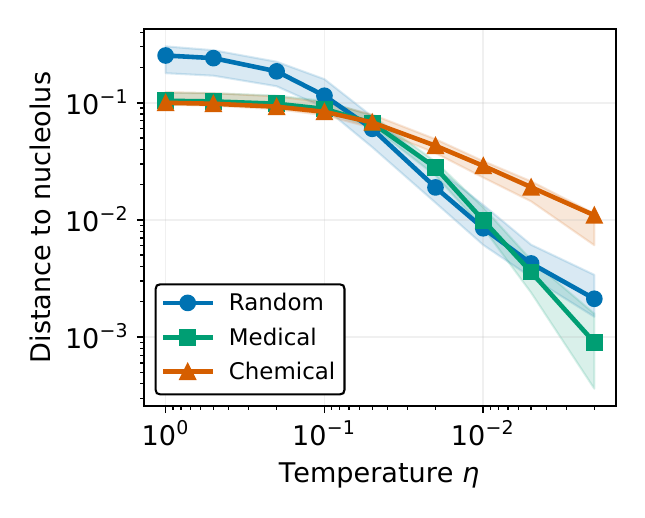}
\caption{Annealing toward the nucleolus: median and interquartile
range for 60 random games, ten medical splits, and nine chemical splits.
Iteration-limited paths are retained; one chemical split lacks a valid
reference and is excluded. Temperature decreases left to right.}
\label{fig:selector-path}
\end{figure}

\section{Computational accuracy and scaling}
\label{app:enumerated-timing-results}

We detail the accuracy and timing comparisons in \cref{sec:selector},
then examine GPU acceleration and coalition sampling at larger sizes.

\subsection{Accuracy and runtime relative to the nucleolus}
\label{app:speed-claims}

The comparison in \cref{tab:nucleolus-speed} measures how cheaply ELC
approximates the nucleolus. It uses 20 random games per size $n=6,7,8$
and ten splits each of medical, chemical, and house, with three timing
repetitions per game. Distances and gaps are medians of within-game
repetition medians; times average all finished repetitions.

ELC time covers all nine annealing stages; nucleolus time covers the
initial least-core LP and subsequent lexicographic optimization.
Both exclude shared game construction and independent accuracy
evaluation. This compares attained accuracy and runtime, rather than
time to a common distance target. The largest absolute least-core gap
of the nucleolus references is $3.78\times10^{-15}$.

House distances use saved references for all ten splits, but its
nucleolus runtime uses only 15 completed repetitions; 15 worker
timeouts under a 3600s cap are excluded. One chemical split is excluded
from min-norm distances because its reference solve failed.
ELC trajectories reaching a stage iteration limit remain included.

\subsection{CPU/GPU scaling and sampled baselines}
\label{app:full-constraint-timing}

\paragraph{Setup.}
Table~\ref{tab:gpu-speed} uses the hardware and solvers in
\cref{sec:selector,app:experiment-protocols}, with identical float64
updates for CPU/GPU ELC and a 600s solver limit.
For $n\in\{8,10,\ldots,24\}$, each game has $n-2$ useful players,
two distractors, $\tau=n-2$, $a_j=0.1(n-2)$, and $b=h_j=2(n-2)$.
Useful effects are uniform in $[0.8,1.2]$, rescaled to sum to $n-2$.
We use five games and three timing repetitions per condition.
The 10\%/50\% families are nested uniform samples with at least $4n$
rows, shared across methods and repetitions.
For each constraint family $\mathcal A$, min-norm selection uses
$e_S(x)\le t_{\mathcal A}+10^{-9}q_{\mathcal A}$, where
$q_{\mathcal A}=\max(1,|v(N)|,\max_{S\in\mathcal A}|v(S)|)$ and
$t_{\mathcal A}$ is the maximum excess of the verified initial LP allocation.

\paragraph{Reading the table.}
Time averages all 15 runs: measured solver times for successes and
numerical failures, and 600s for timeouts ($+$ denotes a lower bound).
It includes the initial LP and internal verification for min-norm,
but excludes data preparation, oracle evaluation, transfers, warm-up,
and independent final evaluation. Num.\ err.\ counts numerical solver failures
or rejected numerical solutions. Gap and AP average successful runs only;
a dash means none succeeded.
Gap is $\max(0,\etahat(x)-\epsilon^\star_{\rm full})/q_{\mathcal S}$,
evaluated on the full family; AP follows \cref{sec:identification}.
Successful subsets may differ across methods: both sampled min-norm
conditions at $n=20$ and the 10\% condition at $n=22$ represent only one
successful game, repeated three times.

\paragraph{Results.}
GPU ELC is $\ScalingSpeedupTwenty{}\times$ and
$\ScalingSpeedupTwentyTwo{}\times$ faster than CPU ELC at $n=20,22$,
respectively, where both complete all runs. At $n=24$, GPU ELC finishes
all runs while CPU ELC always times out. These comparisons use attained
accuracies, not a common accuracy target.
Sampled LP is faster than full LP but has larger full-family gaps;
adding min-norm selection often loses the runtime advantage.
For $n\in\{20,22,24\}$ across both sampling budgets, the initial LP
succeeds in all 90 sampled min-norm runs. All 27 runs containing the
clean coalition $U$ (all useful players and no distractors) succeed;
without $U$, 60 runs time out during SLSQP-based selection and three
report SLSQP numerical failures. These are nine and 21 sampled game
settings, respectively, each repeated three times. This association
localizes the failures to the selection stage but does not establish
the numerical cause or an intrinsic disadvantage of the min-norm selector.

\begingroup
\small\setlength{\tabcolsep}{3pt}
\input{tables/speed_scaling_detail}
\endgroup

\section{Generalized contextual-distractor experiments}
\label{app:sampling-protocol}

We supplement the player-identification results in \cref{sec:identification}
with game generation and evaluation details. Additional tables and figures
separate equal-harm games and show how identification changes with the
context threshold.

\subsection{Games and sampling}
\label{app:generalized-game-protocol}

The games in \cref{tab:strength} follow \cref{eq:freerider-game}
with $(n,|U|,|D|)=(10,6,2)$ and
$(14,8,3)$, leaving two and three neutral players, respectively.
Useful effects are drawn
independently from $[0.8,1.2]$ and rescaled to sum to $|U|$.

The main-text experiments use heterogeneous harm: equally spaced
$h_j/b$ from $0.25$ to $0.75$, permuted across distractors.
We additionally test the homogeneous special case in which every
$h_j=b$. In both cases $a_j=0.1|U|$ and $b=2|U|$.
The thresholds are $3,5,6$ for $n=10$ and $4,6,8$ for $n=14$.
Ten game seeds per size and threshold give 60 games for the main
setting and 60 for the homogeneous special case.

Each game uses all $M=2^n-2$ proper coalitions or a uniform sample
without replacement of $\lfloor\rho M\rfloor$ rows for
$\rho\in\{0.1,0.5\}$. Five sample seeds per game give 1,200 sampled
conditions. The two budgets are nested prefixes of the same random
permutation. No targeted rows or minimum-row adjustment are used.
LP, LP + min-norm, ELC, and the Shapley estimator receive identical sampled
rows and values. ELC uses the shared annealing protocol in
\cref{app:small-game-protocol}.

The equal-count variant in \cref{tab:strength-balanced} instead uses
$(n,|U|,|D|)=(10,4,4)$ and $(14,6,6)$, with two neutrals and thresholds
$\{2,3,4\}$ and $\{3,5,6\}$, respectively. It uses the same two harm
settings, sample construction, and solvers. This changes both group
sizes and the contexts in which harm occurs, so differences between
these grids do not isolate a causal effect of group size.

For each grid, all methods return 120 full-game results. Of 1,200
sampled conditions, 32 LP + min-norm solves fail numerical acceptance
and are excluded; the other methods return all 1,200 outputs.
In the equal-count grid, 31 failures occur at 10\% and one at 50\%.
ELC returns finite iterates even when a stage reaches its iteration limit;
in the original grid, all full endpoints and 1,199 sampled endpoints
meet the final-stage gradient tolerance.

\subsection{Evaluation metrics and aggregation}
\label{app:strength}

We measure group identification, negative-credit assignment, and
within-group ordering separately. AP$^+$ and AP$^-$ rank the useful
and distractor groups by $x_i$ and
$-x_i$, using grouped tied scores. Negative credit means $x_i<-10^{-7}$. For both means and SE, sample seeds are averaged
within each game first; games are the independent units. Failed solves
are excluded, with availability specified in \cref{app:generalized-game-protocol}.

\paragraph{Within-group strength order.}
For useful pairs, we compare $m_i-m_k$ with $x_i-x_k$; for distractor
pairs, we compare $(h_j-a_j)-(h_\ell-a_\ell)$ with $-x_j+x_\ell$.
Only pairs with distinct generating strengths enter the correct, tied,
and reversed fractions; order accuracy is the correct fraction.
The tie tolerance is $10^{-7}\max(1,\|s\|_\infty)$ for a group's
strength vector $s$, applied to both target and payoff differences.
Equal-harm distractors have no target order and enter only identification
and sign metrics.

For the full-coalition game in \cref{tab:freerider}, symmetry and
additive translation explain the common $+7.5$ useful-player offset:
subtracting the additive useful contributions leaves a game symmetric
in useful players, so ELC assigns them a common interaction share.
Subtracting this share would violate efficiency. Asymmetric interactions
need not produce a common offset.

\subsection{Shapley estimation on matched coalition samples}
\label{app:sampled-shapley}

We compute sampled Shapley from the identical uniform-without-replacement
coalition rows used by LP, LP + min-norm, and ELC. This controls the available
coalition information across methods. The comparison concerns this
matched-row estimator rather than alternative Shapley sampling schemes.

Let $\mathcal B\subset\mathcal S$ contain
$m$ of the $M=2^n-2$ proper coalitions. For $k=|S|$, define
\[
 c_i(S)=\begin{cases}
 1/[k\binom nk],&i\in S,\\
 -1/[(n-k)\binom nk],&i\notin S.
 \end{cases}
\]
Collecting the coefficients of each coalition value in the marginal
contribution definition gives the full Shapley formula. Since each row
has inclusion probability $m/M$, its inverse-probability estimator is
\[
 \widehat\phi_i=\frac{v(N)-v(\emptyset)}n+
 \frac M m\sum_{S\in\mathcal B}c_i(S)\,[v(S)-v(\emptyset)].
\]
It is unbiased under the specified sampling design. Also,
$\sum_i c_i(S)=0$ for every proper row, so every output sums to
$v(N)-v(\emptyset)$; here $v(\emptyset)=0$. Full enumeration recovers
the exact value. Inputs are only the sampled coalition rows and values
and the two endpoint values.

\subsection{Identification across game settings}
\label{app:threshold-results}

We test how the identification patterns change with group sizes,
equal distractor harm, and the context threshold.
\Cref{tab:strength-homogeneous,tab:strength-balanced-homogeneous} set
$h_j=b$ for every distractor in the original and equal-count grids.
Full ELC and LP + min-norm identify both groups perfectly in both grids,
but sampling lowers AP$^-$ while retaining high AP$^+$. ELC also preserves
useful-player ordering under the full family, whereas min-norm selection
partly equalizes their payoffs.

\begin{table}[t]
\centering\small\setlength{\tabcolsep}{3pt}
\caption{Equal-harm counterpart to \cref{tab:strength}: $h_j=b$ for all
distractors, with 60 games and the same metrics and aggregation.
Higher is better; bold marks the best mean within each budget, including
ties. Harm order is undefined and shown as a dash.}
\label{tab:strength-homogeneous}
\input{tables/strength_homogeneous}
\end{table}

\paragraph{Equal useful and distractor counts.}
\Cref{tab:strength-balanced} checks whether the sampling effect persists
when the two groups are equally numerous. AP$^-$ still falls for Shapley,
ELC, and LP + min-norm; for the latter two, AP$^+$ remains near one.
The game parameters and the limits of comparison with the original grid
are given in \cref{app:generalized-game-protocol}.

\begin{table}[t]
\centering\small\setlength{\tabcolsep}{3pt}
\caption{Equal useful and distractor counts with varying distractor harm.
As in \cref{tab:strength}, $h_j/b$ ranges from $0.25$ to $0.75$ across
distractors, but group sizes are now equal:
$(n,|U|,|D|)=(10,4,4)$ and $(14,6,6)$, with two neutrals.
Means use 60 games; metrics, aggregation, and boldface follow
\cref{tab:strength}. Game generation and exclusion rules are in
\cref{app:generalized-game-protocol}.}
\label{tab:strength-balanced}
\input{tables/strength_balanced}
\end{table}

\begin{table}[t]
\centering\small\setlength{\tabcolsep}{3pt}
\caption{Equal useful and distractor counts with identical distractor harm.
Group sizes match \cref{tab:strength-balanced}, but every distractor has
$h_j=b$, so any one distractor removes the entire bonus.
Means use 60 games with the same metrics, aggregation, and boldface.
Harm order is undefined because all distractors are equally harmful
and is shown as a dash.}
\label{tab:strength-balanced-homogeneous}
\input{tables/strength_balanced_homogeneous}
\end{table}

The threshold comparisons in \cref{fig:sampling-n10,fig:sampling-n14}
show that distractor identification generally worsens as harmful
contexts become rarer, while useful players remain better identified.
Larger samples usually improve recovery but still fall short of full
ELC and LP + min-norm. Full Shapley also loses distractor identification
at higher thresholds despite preserving useful-player rankings.

\begin{figure}[!t]
\centering
\includegraphics[width=.95\linewidth]{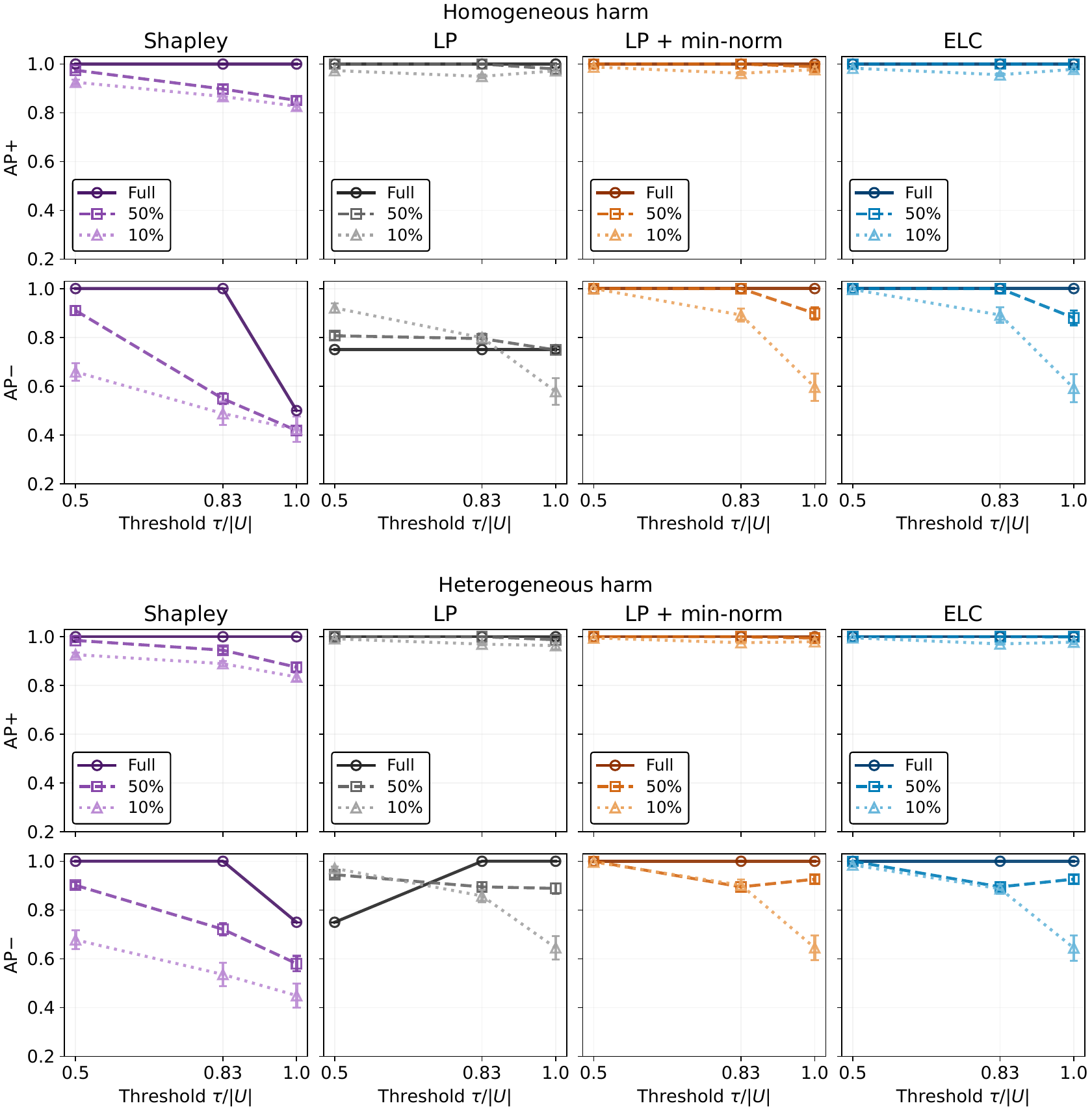}
\caption{Threshold comparisons at $n=10$, with homogeneous harm above
and heterogeneous harm below. Columns compare methods; each pair of
rows shows AP$^+$ and AP$^-$. Full, 50\%, and 10\% coalition families
use dark-to-light shades and distinct line styles and markers.
Sampled Shapley uses the estimator in \cref{app:sampled-shapley}.
Means and SEs use ten games after
averaging samples within each game.}
\label{fig:sampling-n10}
\end{figure}

\begin{figure}[!t]
\centering
\includegraphics[width=.95\linewidth]{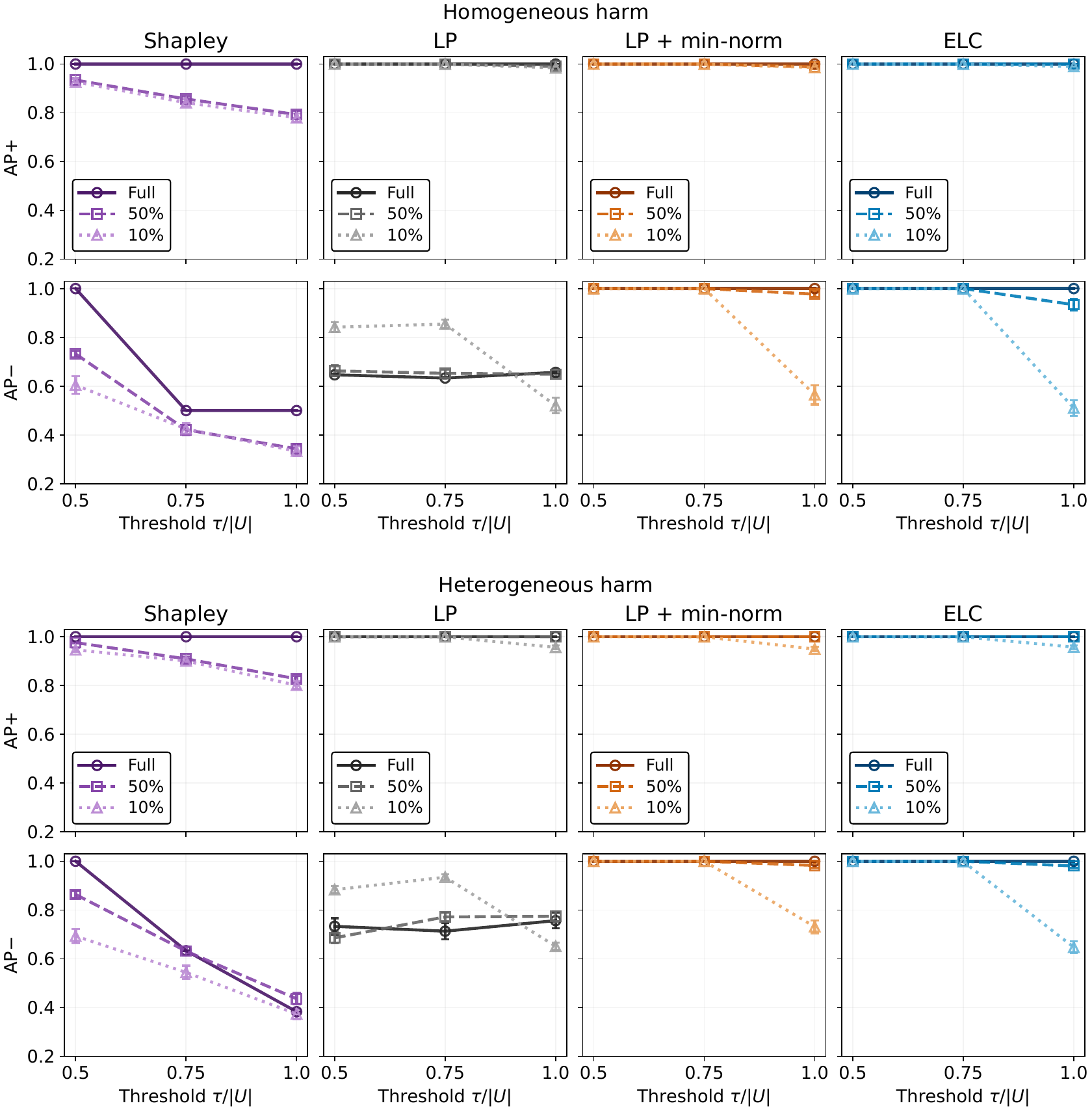}
\caption{Threshold comparisons at $n=14$, with the same layout,
aggregation, and line conventions as \cref{fig:sampling-n10}.}
\label{fig:sampling-n14}
\end{figure}

\section{Downstream attribution checks}
\label{sec:yan-protocol}

These experiments assess whether attribution rankings help with
training-point removal, feature removal, and mislabeled-point detection.
We expand the training-data valuation results in \cref{sec:data-valuation}
and then examine the two additional tasks.

\subsection{Training-point removal curves}
\label{app:data-removal-details}

The Gaussian and Dogfish curves show how prediction accuracy changes
when training points are removed in attribution order
(\cref{fig:gaussian-removal-all,fig:dogfish-removal-all}).
ELC's mean normalized AUC improves on both baselines across all budgets
and removal directions, except against min-norm on Dogfish best-first
removal at 10K queries. The Gaussian gains are larger; several of the
smaller gains have confidence intervals that include zero.

For the paired AUC differences in \cref{tab:yan-data-paired-auc}, we
bootstrap repetition-level differences 10,000 times to obtain 95\%
confidence intervals. The curve bands instead show sample standard
deviations of accuracy, using 20 Gaussian and five Dogfish repetitions.
The two displays therefore summarize different kinds of uncertainty.

\begin{figure}[!t]
\centering
\includegraphics[width=.95\linewidth,height=.76\textheight,keepaspectratio]{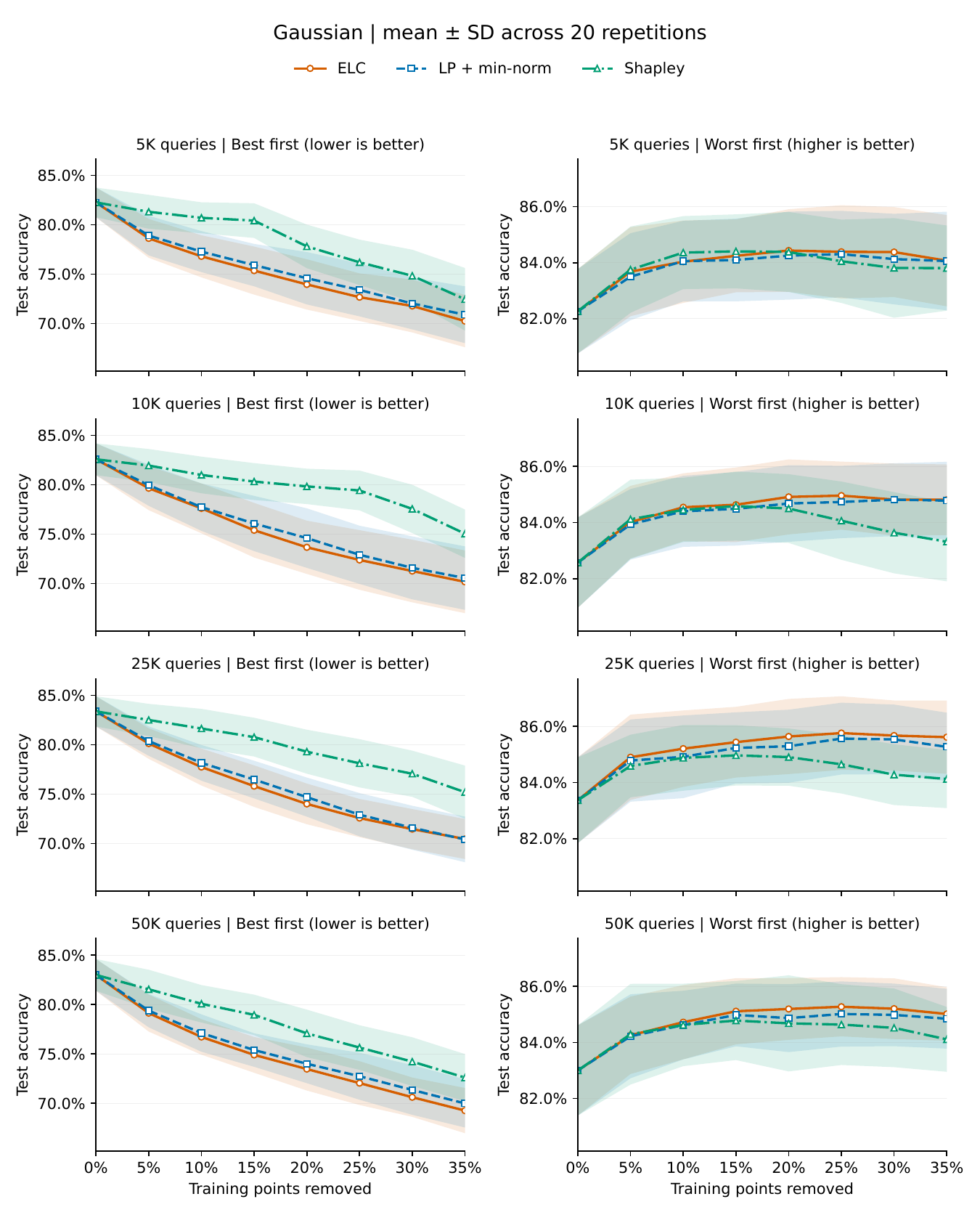}
\caption{Gaussian training-point removal at all four query budgets.
Lines are means over 20 repetitions and bands are $\pm1$ sample standard
deviation. Best-first removal favors lower accuracy; worst-first removal
favors higher accuracy. Vertical scales are shared across budgets within
each removal direction.}
\label{fig:gaussian-removal-all}
\end{figure}

\begin{figure}[!t]
\centering
\includegraphics[width=.95\linewidth,height=.76\textheight,keepaspectratio]{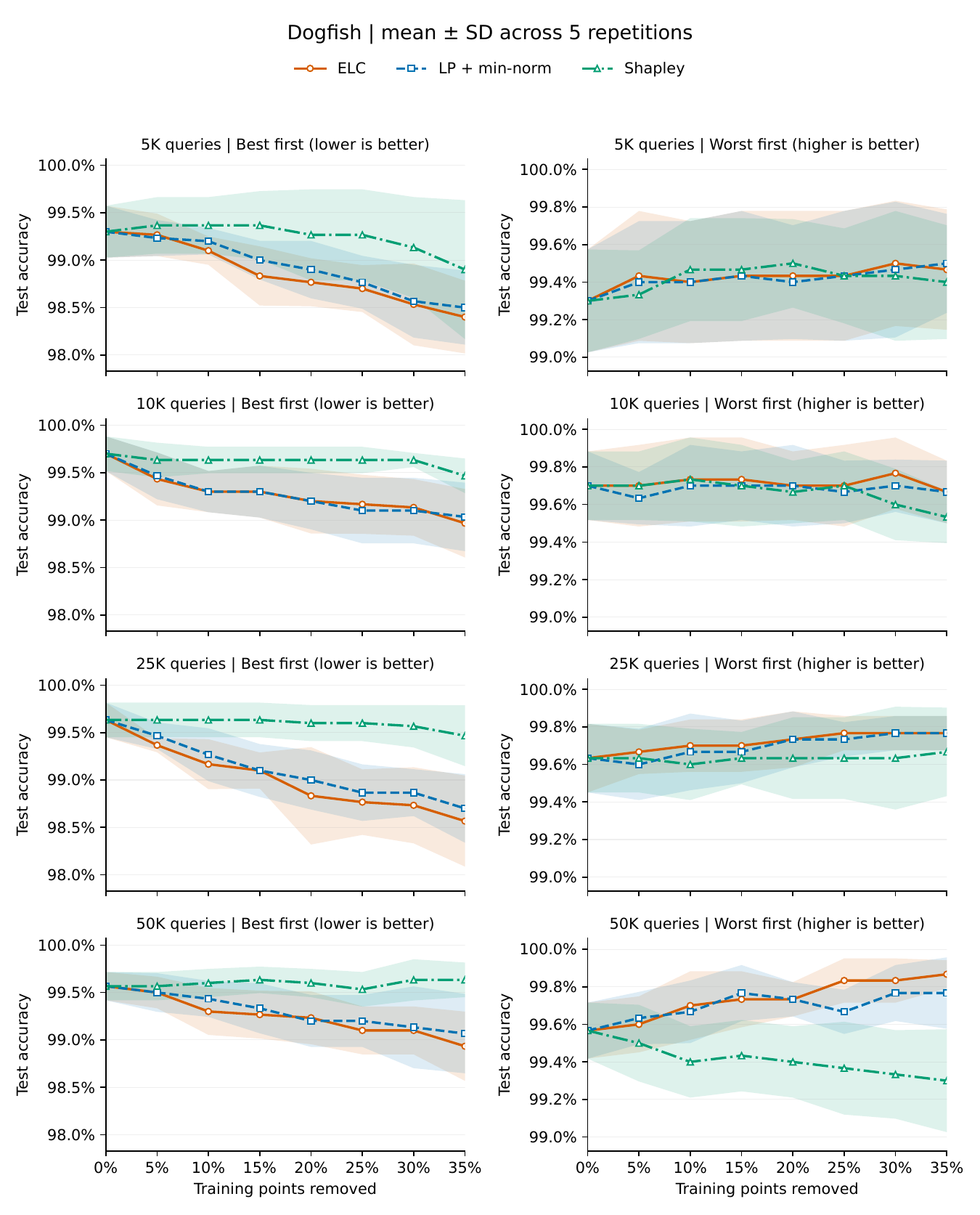}
\caption{Dogfish training-point removal at all four query budgets.
Lines are means over five repetitions and bands are $\pm1$ sample standard
deviation. Layout and scale conventions match \cref{fig:gaussian-removal-all}.}
\label{fig:dogfish-removal-all}
\end{figure}

\subsection{Feature removal on UCI datasets}
\label{app:uci-feature-removal}

This experiment tests feature rankings by removing input features and
refitting a logistic-regression classifier. Coalition values are
held-out accuracies on ten fixed 80/20 stratified splits, with matched
preprocessing across methods. Removal ranges from zero to $n-1$
features; \cref{tab:uci-feature-removal} averages the ELC-minus-min-norm
accuracy difference over these steps and then across splits.

Mean accuracy differences favor ELC in both removal directions on
Chemical. Medical shows almost no mean difference, while the means on
House favor ELC for lowest-first removal and min-norm for highest-first removal. These tasks evaluate predictive
utility rather than recovery of labeled contextual distractors.

\begin{table}[t]
\centering
\caption{Feature-removal accuracy differences, ELC minus (LP + min-norm),
in percentage points: each split contributes its mean over removal
steps, followed by mean $\pm$ SE over ten splits. Positive favors ELC
for lowest-first removal; negative favors ELC for highest-first removal.}
\label{tab:uci-feature-removal}
\input{tables/uci_removal}
\end{table}

\subsection{Label correction on Enron}
\label{app:enron-label-correction}
We adopt the Enron label-correction task of \citet{yan2021} to test
whether low attribution scores identify mislabeled training examples.
We use 1,000 Enron emails with character
$n$-gram count features and Multinomial Naive Bayes. The 489 training
points are players; 20\% of training labels are randomly flipped.
Each of five repetitions uses 5,000 coalition queries, with Shapley
estimated by truncated Monte Carlo.

We restore labels in increasing payoff order. Correction AUC integrates
accuracy over 0--35\% restored points and divides by 0.35.
Mislabel AP ranks $-x_i$ against the flipped-label indicator.
Flip recall is measured among the lowest-ranked points, selecting as
many points as there are flipped labels.

ELC has the highest mean correction AUC, with a small margin over
LP + min-norm and Shapley (\cref{tab:enron-summary}). Shapley is strongest
at identifying flipped labels, while ELC and min-norm have similar
correction curves (\cref{fig:yan-enron-correction}). It is worth noting that mislabeled samples are not necessarily contextual distractors. Contextual distractors are those becoming harmful at some context and otherwise harmless, while mislabeled samples generally do not have such inconsistency. In \citet{yan2021}, the performance of the Shapley value was not reported. 

\input{tables/enron_summary}

\begin{figure}[!t]
\centering
\includegraphics[width=.7\linewidth]{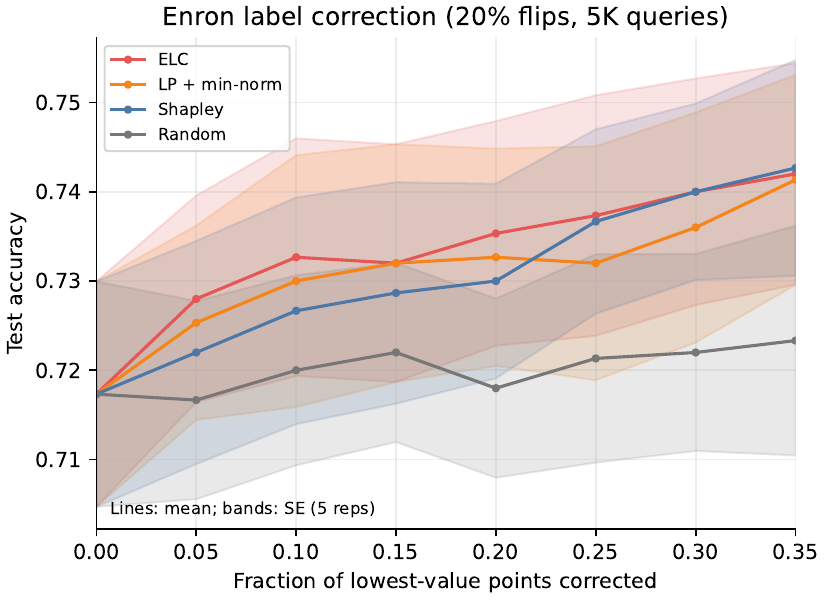}
\caption{Enron label correction after restoring an increasing fraction
of the lowest-valued training points.
Curves are means over five repetitions and bands are mean \(\pm\) SE.}
\label{fig:yan-enron-correction}
\end{figure}

\section{Reference experiment: image-marker attribution}
\label{app:marker}

The image-marker task of \citet{nagarajan2026} tests attribution to a
visually identifiable contextual distractor. A ResNet-50 classifier
predicts cat, dog, or bird; a red marker is added to all dog images and
30\% of cat images during dataset construction. We evaluate 100
correctly classified cats containing the marker, using one sampling seed.
Images are not selected by the sign of the marker's contribution.

Each $224\times224$ image is divided into an $8\times8$ grid of player
patches. The group $D$ comprises patches overlapping the marker.
Coalition values are target-class logits, with absent patches replaced
by black and the empty-image logit subtracted. Removing $D$ therefore
removes entire patches, including pixels outside the marker.

Shapley uses 100 permutations. LP, LP + min-norm, and ELC share 6,400
sampled coalitions with independent patch-inclusion probability $0.5$,
plus the empty and grand coalitions. LP + min-norm uses CVXPY to solve
an LP followed by a quadratic program. ELC uses nine annealing stages
with at most 10,000 updates each. All 100 ELC runs reach the iteration
limit, with final projected-gradient residuals near $10^{-6}$ versus
a $10^{-8}$ threshold. Three min-norm solves fail, leaving 97 images.

The least-core methods assign negative marker credit more than twice
as often as Shapley, including among images where adding the marker
patches reduces the target logit
(\cref{tab:marker-reference,fig:marker-comparison}). ELC and the two
LP-based methods give similar results. This supports the contrast
between least-core and Shapley attribution in this task, without a
separate advantage for ELC's selector.

\begin{table}[t]
\centering
\caption{Marker attribution with one sampling seed. Credit is the sum
$x(D)$ in logit units; negative rate is the fraction with $x(D)<0$.
The last column restricts this rate to images where adding $D$ to its
complement lowers the target logit: 74 images, or 72 among successful
LP + min-norm runs.}
\label{tab:marker-reference}
\small
\input{tables/marker_summary}
\end{table}

\begin{figure}[!t]
\centering
\includegraphics[width=0.65\linewidth]{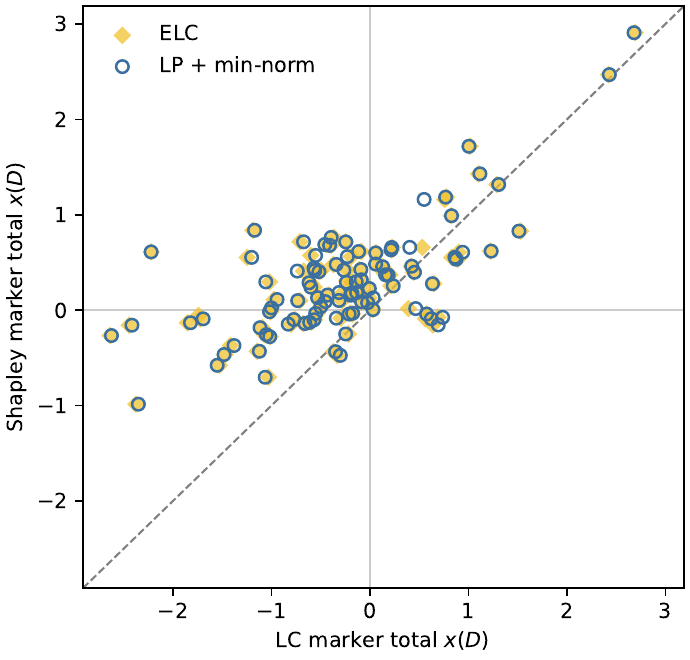}
\caption{Marker attribution: Shapley versus least-core marker credit
$x(D)$, in logit units. Each point represents one image: the horizontal
coordinate is ELC (yellow diamonds, 100 images) or LP + min-norm
(blue open circles, 97 successful solves), and
the vertical coordinate is the corresponding Shapley estimate.
The dashed line denotes equal credit; gray lines mark zero.
The two least-core series largely overlap.}
\label{fig:marker-comparison}
\end{figure}

%% file: tables/initialization.tex
\begin{table}[t]
\centering
\caption{Endpoint variability across five initializations under nine-stage annealing. Diameter is the maximum pairwise endpoint $\ell_2$ distance; hard gap is maximized over the five starts. Values use 20 random games per size. The last two columns count games reaching an iteration limit at any stage or at the final stage, respectively.}
\label{tab:initialization}
\normalsize\setlength{\tabcolsep}{4pt}
\begin{tabular}{rrrrrrr}
\toprule
\(n\) & Games & \multicolumn{2}{c}{Diameter (\(\times10^{-12}\))} & Max hard gap (\(\times10^{-4}\)) & \multicolumn{2}{c}{Capped games} \\
\cmidrule(lr){3-4}\cmidrule(lr){6-7}
 & & Mean \(\pm\) SE & Max & Mean \(\pm\) SE & Any stage & Final stage \\
\midrule
6 & 20 & \(1.269 \pm 0.350\) & 5.594 & \(4.605 \pm 0.998\) & 4 & 0 \\
7 & 20 & \(1.007 \pm 0.382\) & 5.776 & \(7.069 \pm 1.059\) & 4 & 2 \\
8 & 20 & \(2.275 \pm 0.451\) & 5.439 & \(4.836 \pm 1.166\) & 11 & 4 \\
\bottomrule
\end{tabular}
\end{table}

%% file: tables/speed_scaling_detail.tex
\begin{longtable}{rlrrrrrr}
\caption{Computational scaling and attribution on contextual-distractor games. Time averages all 15 runs; Gap and AP average successful runs only. $+$ marks a time lower bound; -- in Gap or AP indicates no successful runs. Settings and aggregation are described in \cref{app:full-constraint-timing}.}\label{tab:scaling-extension}\label{tab:gpu-speed}\\
\toprule
$n$ & Method & Time (s) & Gap & AP$^+$ & AP$^-$ & Timeout/15 & Num.\ err./15 \\
\midrule
\endfirsthead
\multicolumn{8}{c}{Table \thetable\ (continued)} \\
\toprule
$n$ & Method & Time (s) & Gap & AP$^+$ & AP$^-$ & Timeout/15 & Num.\ err./15 \\
\midrule
\endhead
\midrule\multicolumn{8}{r}{Continued on next page} \\
\endfoot
\bottomrule
\endlastfoot
8 & LP & 0.003 & 0 & 1.000 & 1.000 & 0/15 & 0/15 \\*
8 & LP + min-norm & 0.005 & 1e-09 & 1.000 & 1.000 & 0/15 & 0/15 \\*
8 & 50\% LP & 0.003 & 0.2 & 1.000 & 1.000 & 0/15 & 0/15 \\*
8 & 50\% LP + min-norm & 0.004 & 0.2 & 1.000 & 1.000 & 0/15 & 0/15 \\*
8 & 10\% LP & 0.003 & 0.32 & 1.000 & 1.000 & 0/15 & 0/15 \\*
8 & 10\% LP + min-norm & 0.004 & 0.278 & 1.000 & 1.000 & 0/15 & 6/15 \\*
8 & ELC CPU & 0.009 & 2.29e-14 & 1.000 & 1.000 & 0/15 & 0/15 \\*
8 & ELC GPU & 0.029 & 2.12e-14 & 1.000 & 1.000 & 0/15 & 0/15 \\
\midrule
10 & LP & 0.007 & 0 & 1.000 & 1.000 & 0/15 & 0/15 \\*
10 & LP + min-norm & 0.008 & 1e-09 & 1.000 & 1.000 & 0/15 & 0/15 \\*
10 & 50\% LP & 0.006 & 0.0792 & 0.997 & 0.967 & 0/15 & 0/15 \\*
10 & 50\% LP + min-norm & 0.007 & 0.0792 & 1.000 & 1.000 & 0/15 & 0/15 \\*
10 & 10\% LP & 0.003 & 0.28 & 0.995 & 0.917 & 0/15 & 0/15 \\*
10 & 10\% LP + min-norm & 0.004 & 0.28 & 0.995 & 0.917 & 0/15 & 0/15 \\*
10 & ELC CPU & 0.013 & 2.61e-13 & 1.000 & 1.000 & 0/15 & 0/15 \\*
10 & ELC GPU & 0.037 & 2.82e-13 & 1.000 & 1.000 & 0/15 & 0/15 \\
\midrule
12 & LP & 0.024 & 0 & 0.957 & 0.669 & 0/15 & 0/15 \\*
12 & LP + min-norm & 0.026 & 1e-09 & 1.000 & 1.000 & 0/15 & 0/15 \\*
12 & 50\% LP & 0.017 & 0.267 & 0.909 & 0.316 & 0/15 & 0/15 \\*
12 & 50\% LP + min-norm & 0.023 & 0.267 & 0.918 & 0.383 & 0/15 & 0/15 \\*
12 & 10\% LP & 0.006 & 0.276 & 0.918 & 0.396 & 0/15 & 0/15 \\*
12 & 10\% LP + min-norm & 0.008 & 0.276 & 0.918 & 0.383 & 0/15 & 0/15 \\*
12 & ELC CPU & 0.030 & 4.95e-13 & 1.000 & 1.000 & 0/15 & 0/15 \\*
12 & ELC GPU & 0.059 & 4.4e-13 & 1.000 & 1.000 & 0/15 & 0/15 \\
\midrule
14 & LP & 0.116 & 9.87e-18 & 0.866 & 0.578 & 0/15 & 0/15 \\*
14 & LP + min-norm & 0.122 & 1e-09 & 1.000 & 1.000 & 0/15 & 0/15 \\*
14 & 50\% LP & 0.073 & 0.204 & 0.808 & 0.390 & 0/15 & 0/15 \\*
14 & 50\% LP + min-norm & 0.234 & 0.204 & 0.836 & 0.474 & 0/15 & 0/15 \\*
14 & 10\% LP & 0.015 & 0.333 & 0.719 & 0.112 & 0/15 & 0/15 \\*
14 & 10\% LP + min-norm & 0.023 & 0.333 & 0.719 & 0.112 & 0/15 & 0/15 \\*
14 & ELC CPU & 0.088 & 6.46e-15 & 1.000 & 1.000 & 0/15 & 0/15 \\*
14 & ELC GPU & 0.164 & 0 & 1.000 & 1.000 & 0/15 & 0/15 \\
\midrule
16 & LP & 0.527 & 2.54e-17 & 0.870 & 0.567 & 0/15 & 0/15 \\*
16 & LP + min-norm & 0.548 & 1e-09 & 1.000 & 1.000 & 0/15 & 0/15 \\*
16 & 50\% LP & 0.388 & 0.07 & 0.868 & 0.559 & 0/15 & 0/15 \\*
16 & 50\% LP + min-norm & 12.181 & 0.07 & 0.946 & 0.819 & 0/15 & 0/15 \\*
16 & 10\% LP & 0.059 & 0.333 & 0.731 & 0.096 & 0/15 & 0/15 \\*
16 & 10\% LP + min-norm & 0.474 & 0.333 & 0.731 & 0.096 & 0/15 & 0/15 \\*
16 & ELC CPU & 0.408 & 1.45e-15 & 1.000 & 1.000 & 0/15 & 0/15 \\*
16 & ELC GPU & 0.633 & 2.96e-15 & 1.000 & 1.000 & 0/15 & 0/15 \\
\midrule
18 & LP & 2.405 & 0 & 0.879 & 0.559 & 0/15 & 0/15 \\*
18 & LP + min-norm & 2.516 & 1e-09 & 1.000 & 1.000 & 0/15 & 0/15 \\*
18 & 50\% LP & 1.654 & 0.136 & 0.852 & 0.458 & 0/15 & 0/15 \\*
18 & 50\% LP + min-norm & 127.094$^{+}$ & 0.0868 & 0.938 & 0.771 & 3/15 & 0/15 \\*
18 & 10\% LP & 0.267 & 0.333 & 0.751 & 0.095 & 0/15 & 0/15 \\*
18 & 10\% LP + min-norm & 29.617 & 0.333 & 0.751 & 0.095 & 0/15 & 0/15 \\*
18 & ELC CPU & 2.632 & 1.67e-14 & 1.000 & 1.000 & 0/15 & 0/15 \\*
18 & ELC GPU & 2.470 & 0 & 1.000 & 1.000 & 0/15 & 0/15 \\
\midrule
20 & LP & 12.645 & 0 & 0.886 & 0.553 & 0/15 & 0/15 \\*
20 & LP + min-norm & 13.277 & 1e-09 & 1.000 & 1.000 & 0/15 & 0/15 \\*
20 & 50\% LP & 7.134 & 0.267 & 0.791 & 0.176 & 0/15 & 0/15 \\*
20 & 50\% LP + min-norm & 481.904$^{+}$ & 1e-09 & 1.000 & 1.000 & 12/15 & 0/15 \\*
20 & 10\% LP & 1.306 & 0.271 & 0.814 & 0.261 & 0/15 & 0/15 \\*
20 & 10\% LP + min-norm & 432.288$^{+}$ & 0.0196 & 1.000 & 1.000 & 9/15 & 3/15 \\*
20 & ELC CPU & 15.129 & 2.36e-14 & 1.000 & 1.000 & 0/15 & 0/15 \\*
20 & ELC GPU & 9.769 & 4.8e-14 & 1.000 & 1.000 & 0/15 & 0/15 \\
\midrule
22 & LP & 62.547 & 0 & 0.893 & 0.548 & 0/15 & 0/15 \\*
22 & LP + min-norm & 65.457 & 1e-09 & 1.000 & 1.000 & 0/15 & 0/15 \\*
22 & 50\% LP & 44.586 & 0.138 & 0.891 & 0.537 & 0/15 & 0/15 \\*
22 & 50\% LP + min-norm & 272.163$^{+}$ & 0.00723 & 1.000 & 1.000 & 6/15 & 0/15 \\*
22 & 10\% LP & 7.322 & 0.27 & 0.824 & 0.222 & 0/15 & 0/15 \\*
22 & 10\% LP + min-norm & 482.577$^{+}$ & 0.0167 & 1.000 & 1.000 & 12/15 & 0/15 \\*
22 & ELC CPU & 145.965 & 1.58e-14 & 1.000 & 1.000 & 0/15 & 0/15 \\*
22 & ELC GPU & 46.525 & 4.98e-14 & 1.000 & 1.000 & 0/15 & 0/15 \\
\midrule
24 & LP & 317.264 & 0 & 0.898 & 0.543 & 0/15 & 0/15 \\*
24 & LP + min-norm & 330.111 & 1e-09 & 1.000 & 1.000 & 0/15 & 0/15 \\*
24 & 50\% LP & 180.193 & 0.133 & 0.856 & 0.355 & 0/15 & 0/15 \\*
24 & 50\% LP + min-norm & 364.165$^{+}$ & 1e-09 & 1.000 & 1.000 & 6/15 & 0/15 \\*
24 & 10\% LP & 25.628 & 0.333 & 0.793 & 0.063 & 0/15 & 0/15 \\*
24 & 10\% LP + min-norm & 600.000$^{+}$ & -- & -- & -- & 15/15 & 0/15 \\*
24 & ELC CPU & 600.000$^{+}$ & -- & -- & -- & 15/15 & 0/15 \\*
24 & ELC GPU & 214.134 & 0 & 1.000 & 1.000 & 0/15 & 0/15 \\
\end{longtable}

%% file: tables/strength_homogeneous.tex
\begin{tabular}{llrrrrr}
\toprule
Constraints & Method & AP$^+$ & AP$^-$ & Distractor & Useful & Harm \\
 & & & & negative rate & order accuracy & order accuracy \\
\midrule
Full & Shapley & \textbf{1.000} & 0.750 & 0.500 & \textbf{1.000} & \textemdash \\
 & LP & \textbf{1.000} & 0.698 & 0.417 & 0.909 & \textemdash \\
 & LP + min-norm & \textbf{1.000} & \textbf{1.000} & \textbf{1.000} & 0.667 & \textemdash \\
 & ELC & \textbf{1.000} & \textbf{1.000} & \textbf{1.000} & \textbf{1.000} & \textemdash \\
\midrule
50\% sample & Shapley & 0.884 & 0.562 & 0.506 & 0.566 & \textemdash \\
 & LP & 0.995 & 0.719 & 0.482 & \textbf{0.931} & \textemdash \\
 & LP + min-norm & 0.996 & \textbf{0.979} & \textbf{0.993} & 0.676 & \textemdash \\
 & ELC & \textbf{1.000} & 0.969 & \textbf{0.993} & 0.903 & \textemdash \\
\midrule
10\% sample & Shapley & 0.861 & 0.489 & 0.389 & 0.549 & \textemdash \\
 & LP & 0.980 & 0.752 & 0.625 & \textbf{0.807} & \textemdash \\
 & LP + min-norm & \textbf{0.985} & \textbf{0.842} & \textbf{0.794} & 0.775 & \textemdash \\
 & ELC & 0.984 & 0.832 & 0.786 & 0.772 & \textemdash \\
\bottomrule
\end{tabular}

%% file: tables/strength_balanced.tex
\begin{tabular}{llrrrrr}
\toprule
Constraints & Method & AP$^+$ & AP$^-$ & Distractor & Useful & Harm \\
 & & & & negative rate & order accuracy & order accuracy \\
\midrule
Full & Shapley & \textbf{1.000} & 0.673 & 0.208 & \textbf{1.000} & \textbf{1.000} \\
 & LP & \textbf{1.000} & 0.761 & 0.256 & 0.897 & 0.373 \\
 & LP + min-norm & \textbf{1.000} & \textbf{1.000} & \textbf{1.000} & 0.667 & 0.000 \\
 & ELC & \textbf{1.000} & \textbf{1.000} & \textbf{1.000} & \textbf{1.000} & 0.633 \\
\midrule
50\% sample & Shapley & 0.845 & 0.654 & 0.429 & 0.572 & 0.593 \\
 & LP & 0.996 & 0.823 & 0.492 & \textbf{0.860} & 0.334 \\
 & LP + min-norm & 0.995 & \textbf{0.974} & \textbf{0.977} & 0.667 & 0.092 \\
 & ELC & \textbf{1.000} & 0.953 & 0.940 & 0.845 & \textbf{0.620} \\
\midrule
10\% sample & Shapley & 0.787 & 0.609 & 0.267 & 0.564 & 0.550 \\
 & LP & 0.959 & 0.861 & 0.712 & \textbf{0.623} & 0.515 \\
 & LP + min-norm & 0.961 & \textbf{0.891} & \textbf{0.814} & 0.571 & 0.473 \\
 & ELC & \textbf{0.968} & 0.878 & 0.787 & 0.605 & \textbf{0.616} \\
\bottomrule
\end{tabular}

%% file: tables/strength_balanced_homogeneous.tex
\begin{tabular}{llrrrrr}
\toprule
Constraints & Method & AP$^+$ & AP$^-$ & Distractor & Useful & Harm \\
 & & & & negative rate & order accuracy & order accuracy \\
\midrule
Full & Shapley & \textbf{1.000} & 0.764 & 0.167 & \textbf{1.000} & \textemdash \\
 & LP & \textbf{1.000} & 0.757 & 0.208 & 0.896 & \textemdash \\
 & LP + min-norm & \textbf{1.000} & \textbf{1.000} & \textbf{1.000} & 0.667 & \textemdash \\
 & ELC & \textbf{1.000} & \textbf{1.000} & \textbf{1.000} & \textbf{1.000} & \textemdash \\
\midrule
50\% sample & Shapley & 0.804 & 0.603 & 0.370 & 0.574 & \textemdash \\
 & LP & 0.995 & 0.794 & 0.418 & \textbf{0.870} & \textemdash \\
 & LP + min-norm & 0.991 & \textbf{0.976} & \textbf{0.987} & 0.676 & \textemdash \\
 & ELC & \textbf{1.000} & 0.952 & 0.948 & 0.844 & \textemdash \\
\midrule
10\% sample & Shapley & 0.746 & 0.569 & 0.198 & 0.560 & \textemdash \\
 & LP & 0.958 & 0.783 & 0.559 & \textbf{0.726} & \textemdash \\
 & LP + min-norm & 0.961 & \textbf{0.821} & \textbf{0.687} & 0.657 & \textemdash \\
 & ELC & \textbf{0.965} & 0.808 & 0.646 & 0.696 & \textemdash \\
\bottomrule
\end{tabular}

%% file: tables/uci_removal.tex
\begin{tabular}{lrr}
\toprule
Dataset & Lowest first & Highest first \\
\midrule
Medical & $-0.032\pm0.056$ & $+0.000\pm0.068$ \\
Chemical & $+0.299\pm0.214$ & $-0.449\pm0.217$ \\
House & $+0.029\pm0.031$ & $+0.244\pm0.242$ \\
\bottomrule
\end{tabular}

%% file: tables/enron_summary.tex
\begin{table*}[t]
\centering
\caption{Enron label-correction and mislabel-identification results (mean \(\mathbin{\pm}\) SE over five repetitions). Bold and underline mark the highest and second-highest mean in each performance column, respectively.}
\label{tab:enron-summary}
\scriptsize
\resizebox{\textwidth}{!}{%
\begin{tabular}{lrrrr}
\toprule
Method & Corr. norm. AUC & Corr. acc. at 35\% & Mislabel AP & Flip recall \\
\midrule
Random & \(0.7200\mathbin{\pm}0.011\) & \(0.723\mathbin{\pm}0.013\) & \(0.211\mathbin{\pm}0.005\) & \(0.204\mathbin{\pm}0.018\) \\
Shapley & \(0.7306\mathbin{\pm}0.011\) & \(\mathbf{0.743}\mathbin{\pm}0.012\) & \(\mathbf{0.456}\mathbin{\pm}0.023\) & \(\mathbf{0.433}\mathbin{\pm}0.036\) \\
LP + min-norm & \(\underline{0.7310}\mathbin{\pm}0.013\) & \(0.741\mathbin{\pm}0.012\) & \(0.268\mathbin{\pm}0.015\) & \(0.273\mathbin{\pm}0.016\) \\
ELC & \(\mathbf{0.7336}\mathbin{\pm}0.013\) & \(\underline{0.742}\mathbin{\pm}0.012\) & \(\underline{0.274}\mathbin{\pm}0.015\) & \(\underline{0.286}\mathbin{\pm}0.014\) \\
\bottomrule
\end{tabular}
}
\end{table*}

%% file: tables/marker_summary.tex
\begin{tabular}{lrrrr}
\toprule
Method & Images & Mean $x(D)$ & Negative rate & Negative on harmful \\
\midrule
Shapley & 100 & $+0.294$ & 31.0\% & 37.8\% \\
LP & 100 & $-0.259$ & 67.0\% & 77.0\% \\
LP + min-norm & 97 & $-0.250$ & 67.0\% & 76.4\% \\
ELC & 100 & $-0.257$ & 67.0\% & 77.0\% \\
\bottomrule
\end{tabular}

%% file: refs.bib
@incollection{shapley1953,
  author    = {Shapley, L. S.},
  title     = {A Value for {$n$}-Person Games},
  booktitle = {Contributions to the Theory of Games II},
  publisher = {Princeton University Press},
  year      = {1953},
}

@article{maschler1979,
  author  = {Maschler, M. and Peleg, B. and Shapley, L. S.},
  title   = {Geometric Properties of the Kernel, Nucleolus, and Related Solution Concepts},
  journal = {Mathematics of Operations Research},
  volume  = {4},
  number  = {4},
  pages   = {303--338},
  year    = {1979},
  doi       = {10.1287/moor.4.4.303},
}

@article{schmeidler1969,
  author  = {Schmeidler, D.},
  title   = {The Nucleolus of a Characteristic Function Game},
  journal = {SIAM Journal on Applied Mathematics},
  volume  = {17},
  number  = {6},
  pages   = {1163--1170},
  year    = {1969},
  doi       = {10.1137/0117107},
}

@inproceedings{yan2021,
  author    = {Yan, T. and Procaccia, A. D.},
  title     = {If You Like {Shapley} Then You'll Love the Core},
  booktitle = {Proceedings of the 35th AAAI Conference on Artificial Intelligence (AAAI)},
  pages     = {5751--5759},
  year      = {2021},
  volume    = {35},
  number    = {6},
  doi       = {10.1609/aaai.v35i6.16721},
}

@inproceedings{gemp2024,
  title     = {Approximating the Core via Iterative Coalition Sampling},
  booktitle = {Proceedings of the 23rd International Conference on Autonomous Agents and Multiagent Systems (AAMAS)},
  pages     = {669--678},
  year      = {2024},
  author    = {Gemp, Ian and Lanctot, Marc and Marris, Luke and Mao, Yiran and Du{\'e}{\~n}ez-Guzm{\'a}n, Edgar and Perrin, Sarah and Gyorgy, Andras and Elie, Romuald and Piliouras, Georgios and Kaisers, Michael and Hennes, Daniel and Bullard, Kalesha and Larson, Kate and Bachrach, Yoram},
  url       = {https://www.ifaamas.org/Proceedings/aamas2024/pdfs/p669.pdf},
}

@article{dubey1981,
  author  = {Dubey, P. and Neyman, A. and Weber, R. J.},
  title   = {Value Theory without Efficiency},
  journal = {Mathematics of Operations Research},
  volume  = {6},
  number  = {1},
  pages   = {122--128},
  year    = {1981},
  doi       = {10.1287/moor.6.1.122},
}

@inproceedings{nagarajan2026,
  author    = {Nagarajan, S. G. and Yamauchi, T. and Kera, H.},
  title     = {Seeing through Distractions: Stable Attribution via the Core},
  booktitle = {NExT-Game Workshop at ICML},
  year      = {2026},
}

@book{nesterov2018,
  author    = {Nesterov, Y.},
  title     = {Lectures on Convex Optimization},
  edition   = {2},
  publisher = {Springer},
  year      = {2018},
  series    = {Springer Optimization and Its Applications},
  volume    = {137},
  doi       = {10.1007/978-3-319-91578-4},
}

@article{alvarez2000,
  author  = {Alvarez, F.},
  title   = {Absolute Minimizer in Convex Programming by Exponential Penalty},
  journal = {Journal of Convex Analysis},
  volume  = {7},
  number  = {1},
  pages   = {197--202},
  year    = {2000},
}

@inproceedings{lundberg2017,
  author    = {Lundberg, Scott M. and Lee, Su-In},
  title     = {A Unified Approach to Interpreting Model Predictions},
  booktitle = {Advances in Neural Information Processing Systems (NeurIPS)},
  year      = {2017},
  volume    = {30},
  url       = {https://proceedings.neurips.cc/paper/2017/hash/8a20a8621978632d76c43dfd28b67767-Abstract.html},
}

@article{strumbelj2010,
  author  = {{\v{S}}trumbelj, Erik and Kononenko, Igor},
  title   = {An Efficient Explanation of Individual Classifications Using Game Theory},
  journal = {Journal of Machine Learning Research},
  volume  = {11},
  pages   = {1--18},
  year    = {2010},
  number    = {1},
  url       = {https://www.jmlr.org/papers/v11/strumbelj10a.html},
}

@inproceedings{sundararajan2017,
  author    = {Sundararajan, Mukund and Taly, Ankur and Yan, Qiqi},
  title     = {Axiomatic Attribution for Deep Networks},
  booktitle = {Proceedings of the 34th International Conference on Machine Learning (ICML)},
  pages     = {3319--3328},
  year      = {2017},
  volume    = {70},
  series    = {Proceedings of Machine Learning Research},
  publisher = {PMLR},
  url       = {https://proceedings.mlr.press/v70/sundararajan17a.html},
}

@inproceedings{sundararajan2020,
  author    = {Sundararajan, Mukund and Najmi, Amir},
  title     = {The Many {Shapley} Values for Model Explanation},
  booktitle = {Proceedings of the 37th International Conference on Machine Learning (ICML)},
  pages     = {9269--9278},
  year      = {2020},
  volume    = {119},
  series    = {Proceedings of Machine Learning Research},
  publisher = {PMLR},
  url       = {https://proceedings.mlr.press/v119/sundararajan20b.html},
}

@inproceedings{ghorbani2019,
  author    = {Ghorbani, Amirata and Zou, James},
  title     = {Data {Shapley}: Equitable Valuation of Data for Machine Learning},
  booktitle = {Proceedings of the 36th International Conference on Machine Learning (ICML)},
  year      = {2019},
  volume    = {97},
  series    = {Proceedings of Machine Learning Research},
  publisher = {PMLR},
  url       = {https://proceedings.mlr.press/v97/ghorbani19c.html},
  pages     = {2242--2251},
}

@article{aas2021,
  author  = {Aas, Kjersti and Jullum, Martin and L{\o}land, Anders},
  title   = {Explaining Individual Predictions When Features Are Dependent: More Accurate Approximations to {Shapley} Values},
  journal = {Artificial Intelligence},
  volume  = {298},
  pages   = {103502},
  year    = {2021},
}

@inproceedings{covert2020,
  author    = {Covert, Ian and Lundberg, Scott M. and Lee, Su-In},
  title     = {Understanding Global Feature Contributions with Additive Importance Measures},
  booktitle = {Advances in Neural Information Processing Systems (NeurIPS)},
  year      = {2020},
  volume    = {33},
  url       = {https://proceedings.neurips.cc/paper_files/paper/2020/hash/c7bf0b7c1a86d5eb3be2c722cf2cf746-Abstract.html},
}

@inproceedings{jethani2022,
  author    = {Jethani, Neil and Sudarshan, Mukund and Covert, Ian and Lee, Su-In and Ranganath, Rajesh},
  title     = {{FastSHAP}: Real-Time {Shapley} Value Estimation},
  booktitle = {International Conference on Learning Representations (ICLR)},
  year      = {2022},
}

@article{mitchell2022,
  author  = {Mitchell, Rory and Cooper, Joshua and Frank, Eibe and Holmes, Geoffrey},
  title   = {Sampling Permutations for {Shapley} Value Estimation},
  journal = {Journal of Machine Learning Research},
  volume  = {23},
  number  = {43},
  pages   = {1--46},
  year    = {2022},
}

@inproceedings{sumiyasu2024,
  author    = {Sumiyasu, Kosuke and Kawamoto, Kazuhiko and Kera, Hiroshi},
  title     = {Identifying Important Group of Pixels Using Interactions},
  booktitle = {Proceedings of the IEEE/CVF Conference on Computer Vision and Pattern Recognition (CVPR)},
  pages     = {6017--6026},
  year      = {2024},
}

@misc{sumiyasu2022,
  author        = {Sumiyasu, Kosuke and Kawamoto, Kazuhiko and Kera, Hiroshi},
  title         = {Game-Theoretic Understanding of Misclassification},
  year          = {2022},
  howpublished  = {arXiv preprint arXiv:2210.03349},
  eprint        = {2210.03349},
  archivePrefix = {arXiv},
  primaryClass  = {cs.CV},
}

@inproceedings{chang2025,
  author    = {Chang, Wonjoon and Lee, Myeongjin and Choi, Jaesik},
  title     = {Rethinking {Shapley} Value for Negative Interactions in Non-Convex Games},
  booktitle = {International Conference on Learning Representations (ICLR)},
  year      = {2025},
}

@inproceedings{yamauchi2026,
  author    = {Yamauchi, Toshinori and Kera, Hiroshi and Kawamoto, Kazuhiko},
  title     = {Explaining Object Detectors via Collective Contribution of Pixels},
  booktitle = {Proceedings of the IEEE/CVF Conference on Computer Vision and Pattern Recognition (CVPR)},
  pages     = {17046--17056},
  year      = {2026},
}

@inproceedings{cuturi2013sinkhorn,
author = {Cuturi, Marco},
year = {2013},
publisher = {Curran Associates Inc.},
  title     = {{Sinkhorn} Distances: Lightspeed Computation of Optimal Transport},
  booktitle = {Advances in Neural Information Processing Systems},
  volume    = {26},
  pages     = {2292--2300},
  url       = {https://papers.neurips.cc/paper_files/paper/2013/hash/af21d0c97db2e27e13572cbf59eb343d-Abstract.html},
}

@inproceedings{abernethy2024,
  author = {Abernethy, Jacob A. and Schapire, Robert E. and Syed, Umar},
  title = {Lexicographic Optimization: Algorithms and Stability},
  booktitle = {Proceedings of the 27th International Conference on Artificial Intelligence and Statistics},
  series = {Proceedings of Machine Learning Research},
  volume = {238},
  pages = {2503--2511},
  year = {2024},
  url = {https://proceedings.mlr.press/v238/abernethy24a.html},
}

@article{castro2009polynomial,
  author={Castro, Javier and G{\'o}mez, Daniel and Tejada, Juan},
  volume={36},
  number={5},
  pages={1726--1730},
  year={2009},
  publisher={Elsevier},
  title     = {Polynomial calculation of the {Shapley} value based on sampling},
  journal   = {Computers \& Operations Research},
  doi       = {10.1016/j.cor.2008.04.004},
}

@article{grabisch1999an,
  author  = {Grabisch, Michel and Roubens, Marc},
  title   = {An Axiomatic Approach to the Concept of Interaction among Players in Cooperative Games},
  journal = {International Journal of Game Theory},
  volume  = {28},
  pages   = {547--565},
  year    = {1999},
  doi       = {10.1007/s001820050125},
}

@inproceedings{kwon2022beta,
  author={Kwon, Yongchan and Zou, James},
  title     = {Beta {Shapley}: a Unified and Noise-reduced Data Valuation Framework for Machine Learning},
  booktitle = {Proceedings of The 25th International Conference on Artificial Intelligence and Statistics},
  year      = {2022},
  volume    = {151},
  series    = {Proceedings of Machine Learning Research},
  pages     = {8780--8802},
  publisher = {PMLR},
  url       = {https://proceedings.mlr.press/v151/kwon22a.html},
}

@inproceedings{wang2023data,
  pages={6388--6421},
  year={2023},
  author    = {Wang, Jiachen T. and Jia, Ruoxi},
  title     = {Data {Banzhaf}: A Robust Data Valuation Framework for Machine Learning},
  booktitle = {Proceedings of The 26th International Conference on Artificial Intelligence and Statistics},
  volume    = {206},
  series    = {Proceedings of Machine Learning Research},
  publisher = {PMLR},
  url       = {https://proceedings.mlr.press/v206/wang23e.html},
}

@article{pred-diff,
    journal = {Artificial Intelligence},
    volume = {312},
    pages = {103774},
    year = {2022},
    author = {Stefan Blücher and Johanna Vielhaben and Nils Strodthoff},
  title     = {{PredDiff}: Explanations and interactions from conditional expectations},
  doi       = {10.1016/j.artint.2022.103774},
}

@inproceedings{li2026nonlinear,
  title     = {Nonlinear Axiomatic Attribution for Cooperative Games},
  author    = {Li, Weida and Liu, Zhuanghua and Yu, Yaoliang and Low, Bryan Kian Hsiang},
  booktitle = {Proceedings of the 42nd Conference on Uncertainty in Artificial Intelligence},
  pages     = {3581--3603},
  year      = {2026},
  volume    = {337},
  series    = {Proceedings of Machine Learning Research},
  publisher = {PMLR},
  url       = {https://proceedings.mlr.press/v337/li26c.html},
}

@inproceedings{harris2026core,
  title     = {In-context credit assignment via the core},
  author    = {Harris, Keegan and Prasad, Siddharth and Trockman, Asher},
  booktitle = {{NExT-Game Workshop at ICML}},
  year      = {2026},
  url       = {https://arxiv.org/abs/2605.06920},
}

@book{grotschellovaszschrijver1988,
  author    = {Gr{\"o}tschel, Martin and Lov{\'a}sz, L{\'a}szl{\'o} and Schrijver, Alexander},
  title     = {Geometric Algorithms and Combinatorial Optimization},
  series    = {Algorithms and Combinatorics},
  volume    = {2},
  publisher = {Springer-Verlag},
  address   = {Berlin, Heidelberg},
  year      = {1988},
  doi       = {10.1007/978-3-642-97881-4}
}

@article{deng1994complexity,
  author  = {Deng, Xiaotie and Papadimitriou, Christos H.},
  title   = {On the Complexity of Cooperative Solution Concepts},
  journal = {Mathematics of Operations Research},
  volume  = {19},
  number  = {2},
  pages   = {257--266},
  year    = {1994},
  doi     = {10.1287/moor.19.2.257}
}
